\documentclass[10pt]{article}
\usepackage{amsmath,amsfonts,amsthm,mathrsfs,url}
\usepackage{color}
\usepackage{fancyhdr}
\usepackage[top=2.3cm, bottom=2.5cm, left=2.1cm,
right=2.1cm]{geometry}
\usepackage{bm}
\usepackage{enumitem,amsthm,amsmath,amsfonts,amssymb}
\usepackage{float}
\usepackage{mathtools}
\usepackage{algorithmic}
\usepackage{algorithm}
\usepackage{xcolor}
\usepackage{graphicx}
\usepackage{subcaption}
\usepackage{tikz}
\usepackage{multirow}
\usepackage{multicol}
\usepackage{hyperref}
\usepackage{mathrsfs, fancyhdr}
\usepackage{multirow}
\usepackage{booktabs}

\newcommand{\Z}{\mathbb{Z}}

\newcommand{\N}{\mathbb{N}}

\newcommand{\E}{\mathbb{E}}

\newcommand{\R}{\mathbb{R}}

\newtheorem{cor}{Corollary}

\newtheorem{prop}{Proposition}
\newtheorem{rem}{Remark}

\def\bw{\mathbf{w}}
\def\O{\mathcal{O}}
\def\gep{\epsilon}

\def\X{\mathcal{X}}
\def\Y{\mathcal{Y}}
\def\W{\mathcal{W}}

\def\Z{\mathcal{Z}}
\def\N{\mathcal{N}}
\def\bx{\mathbf{x}}

\def\bw{\mathbf{w}}

\def\priv{\text{priv}}

\def\bb{\mathbf{b}}
\def\A{\mathcal{A}}

\def\ga{\alpha}
\def\gd{\delta}
\def\gep{\epsilon}
\def\gD{\Delta}

\def\0{\mathbf{0}}

\def\proj{\text{Proj}}

\newtheorem{theorem}{Theorem}
\newtheorem{lemma}[theorem]{Lemma}
\newtheorem{proposition}[theorem]{Proposition}
\newtheorem{corollary}[theorem]{Corollary}
\theoremstyle{definition}
\newtheorem{definition}{Definition}

\newtheorem{example}{Example}
\theoremstyle{definition}

\newtheorem{remark}{Remark}

\def\begeqn{\begin{equation}}
\def\endeqn{\end{equation}}
\def\begth{\begin{theorem}}
\def\endth{\end{theorem}}
\def\begprop{\begin{proposition}}
\def\endprop{\end{proposition}}
\def\begcor{\begin{corollary}}
\def\endcor{\end{corollary}}
\def\begdef{\begin{definition}}
\def\enddef{\end{definition}}
\def\beglemm{\begin{lemma}}
\def\endlemm{\end{lemma}}
\def\begexm{\begin{example}}
\def\endexm{\end{example}}
\def\begrem{\begin{remark}}
\def\endrem{\end{remark}}

\def\begdef{\begin{definition}}
\def\enddef{\end{definition}}

\def\proj{\text{Proj}}

\begin{document}

\title{Differentially Private Stochastic Gradient Descent with Low-Noise}
\author{Puyu Wang$^{1}$, Yunwen Lei$^{2}$, Yiming Ying$^{3*}$ and Ding-Xuan Zhou$^{4}$\\ \\ $^{1}$ Liu Bie Ju Centre for Mathematical Sciences, City University of Hong Kong, Hong Kong \\
$^{2}$ Department of Mathematics, The University of Hong Kong, HK\\
$^{3}$ Department of Mathematics and Statistics, State University of New York at Albany, USA\\
$^{4}$ School of Mathematics and Statistics, University of Sydney, Australia}

 \date{}

\maketitle

\begin{abstract} 
Modern machine learning algorithms aim to extract fine-grained information from data to provide accurate predictions, which often conflicts with the goal of privacy protection. This paper addresses the practical and theoretical importance of developing privacy-preserving machine learning algorithms that ensure good performance while preserving privacy. 
In this paper,  we focus on the privacy and utility (measured by excess risk bounds) performances of differentially private stochastic gradient descent (SGD) algorithms in the setting of stochastic convex optimization. Specifically, we examine the pointwise problem in the low-noise setting for which we derive sharper excess risk bounds for the differentially private SGD algorithm. 
 In the pairwise learning setting, we propose a simple differentially private SGD algorithm based on gradient perturbation. Furthermore, we develop novel utility bounds for the proposed algorithm, proving that it achieves optimal excess risk rates even for non-smooth losses. Notably, we establish fast learning rates for privacy-preserving pairwise learning under the low-noise condition, which is the first of its kind.  

\

\end{abstract}

\bigskip

\parindent=0cm

\section{Introduction}
 Stochastic gradient descent (SGD) iteratively updates model parameters using the gradient information over a small batch of random examples, which reduces the computation cost and makes it amenable to solving large-scale problems.  Due to its low computational overhead and easy implementation, it has become the workhorse algorithm for training many machine learning models \cite{duchi2009efficient,goodfellow2016deep,li2019convergence,li2018learning,lin2017optimal,rakhlin2012making,roux2012stochastic,yang2021simple,zhang2004solving}. 
 
On the other important front, we have witnessed a significant risk of privacy leakage by sharing gradient information of machine learning models because the gradient often embeds knowledge about the training data.  For instance, \cite{zhu2019deep} provides paradigms for breaching privacy and reconstructing training examples from publicly shared gradients and \cite{shokri2017membership} shows that the membership of a data record can be inferred from a binary classifier trained on gradients. As SGD is widely deployed in machine learning models, it is crucial to develop private SGD algorithms to mitigate the privacy leakage posted by gradients.

In this paper, we are interested in differentially private SGD (DP-SGD)  for both pointwise and pairwise learning problems.  Differential privacy (DP) \cite{dwork2006calibrating} is a de facto concept for designing private algorithms, which defines a rigorous attack model independent of background knowledge and gives a quantitative representation of the degree of privacy leakage.  There is a considerable amount of work \cite{bassily2020stability,bassily2019private,bassily2014private,bassily2021differentially,feldman2020private,su2022faster,wang2022differentially,wang2021differentially,xue2021differentially,yang2021simple} on analyzing the utility guarantee (i.e., statistical generalization performance) of DP-SGD algorithms. In particular,  \cite{bassily2020stability,bassily2014private,feldman2020private,wang2022differentially,yang2021simple} have shown that private SGD algorithms can achieve the optimal excess population risk bound  $\O\big(\frac{1}{\sqrt{n}}+ \frac{1}{n\epsilon}\sqrt{d\log(1/\delta)} \big)$  for solving convex problems  in different settings. Here, $n$ is the size of the training dataset, $d$ is the dimension, and $(\epsilon,\delta)$ are privacy parameters.  
One nature question then arises: can DP-SGD algorithms  achieve faster utility rates beyond  $\O\big(\frac{1}{\sqrt{n}}+ \frac{1}{n\epsilon}\sqrt{d\log(1/\delta)} \big)$?

\begin{table*}[!t]\label{table:summary}
{\small{ 
\caption{Comparison of different $(\epsilon,\delta)$-DP algorithms for pointwise learning. 
Here, $\alpha$-H\"older denotes $\alpha$-H\"older smooth losses.  
}
\begin{center}
\begin{tabular}{ c|cccccc } 
 \hline
 \multirow{1}{*}{\textbf{Work} }     &   \multirow{1}{*}{ \textbf{Lipschitz}}& \multirow{1}{*}{ \textbf{Smooth}} &  \multirow{1}{*}{ \textbf{Low-noise}} & \multirow{1}{*}{ \textbf{Gradient complexity}}  & \multirow{1}{*}{ \textbf{Utility }}  \\ 
 \hline
 \multirow{4}{*}{\cite{bassily2019private}}     &  \multirow{2}{*}{ $\checkmark$ }  &  \multirow{2}{*}{$\checkmark$}&  \multirow{2}{*}{$\times$} &  \multirow{2}{*}{$\O\big( n^{1.5} \sqrt{\epsilon} + (n\epsilon)^{2.5} ( d \log(1/{\delta}))^{-1} \big)$} &  \multirow{2}{*}{$\O\big(\frac{1}{\sqrt{n}}+\frac{1}{n\epsilon}\sqrt{d\log(1/\delta) }   \big)$} \\
  &&& &&&\\
  ~   &  \multirow{2}{*}{$\checkmark$}  &  \multirow{2}{*}{$\times$}   &  \multirow{2}{*}{$\times$}  &  \multirow{2}{*}{$\O\big( n^{4.5} \sqrt{\epsilon} + n^{6.5}\epsilon^{4.5}{( d \log(1/{\delta}))^{-2}} \big)$}  &  \multirow{2}{*}{$\O\big(\frac{1}{\sqrt{n}}+\frac{1}{n\epsilon}\sqrt{d\log(1/\delta) }   \big)$} \\  
 &&&&&&\\
  \hline
 \multirow{2}{*}{\cite{bassily2020stability}}    &  \multirow{2}{*}{$\checkmark$}  &  \multirow{2}{*}{$\times$}  &  \multirow{2}{*}{$\times$} &  \multirow{2}{*}{$\O(n^2)$}  &  \multirow{2}{*}{$\O\big(\frac{1}{\sqrt{n}}+\frac{1}{n\epsilon}\sqrt{d\log(1/\delta) }  \big)$}  \\ 
  &&&&&&\\
 \hline
 \multirow{4}{*}{\cite{wang2022differentially}}   &  \multirow{2}{*}{$\times$}&  \multirow{2}{*}{$\checkmark$}&  \multirow{2}{*}{$\times$}&  \multirow{2}{*}{$\O(n )$} &  \multirow{2}{*}{$\O\big(\frac{1}{\sqrt{n}}+\frac{1}{n\epsilon}\sqrt{d\log(1/\delta) }  \big)$}  \\ 
 & &&&\\
 ~ &\multirow{2}{*}{\textit{$\times$}}   &  \multirow{2}{*}{$\alpha$-H\"older }&  \multirow{2}{*}{$\times$}  &  \multirow{2}{*}{ $\O\big(n^{
\frac{2-\alpha}{1+\alpha}}+n\big)$} &  \multirow{2}{*}{$\O\big(\frac{1}{\sqrt{n}}+\frac{1}{n\epsilon}\sqrt{d\log(1/\delta) }  \big)$} \\
   ~& & &&\\
   \hline
 \multirow{7}{*}{Ours}     &  \multirow{2}{*}{$\checkmark$}&  \multirow{2}{*}{$\checkmark$} &  \multirow{2}{*}{$\times$}&  \multirow{2}{*}{$\O(n)$}&  \multirow{2}{*}{$\O\big(\frac{1}{\sqrt{n}}+\frac{1}{n\epsilon}\sqrt{d\log(1/\delta) } \big)$}  \\ 
  ~ &&&&&&\\
 ~  &  \multirow{1.5}{*}{$\checkmark$} &  \multirow{1.5}{*}{$\checkmark$}&  \multirow{1.5}{*}{$\checkmark$}&  \multirow{1.5}{*}{$\O(n )$} &  \multirow{1.5}{*}{$\O\big( \frac{1}{n\epsilon}\sqrt{d\log(1/\delta)}\big)$} \\
 ~  &  \multirow{3}{*}{$\checkmark$} &  \multirow{3}{*}{$\alpha$-H\"older}&  \multirow{3}{*}{$\times$}&  \multirow{3}{*}{$\O\big(n^{
\frac{2-\alpha}{1+\alpha}}+n\big)$}  &  \multirow{3}{*}{$\O\big(\frac{1}{\sqrt{n}}+\frac{1}{n\epsilon}\sqrt{d\log(1/\delta)}\big)$} \\
 & & &&\\
 ~   &  \multirow{2}{*}{$\checkmark$} &  \multirow{2}{*}{$\alpha$-H\"older}&  \multirow{2}{*}{$\checkmark$}&  \multirow{2}{*}{$\O\big(n^{
\frac{2}{1+\alpha}}\big)$} &  \multirow{2}{*}{$\O\big( {n^{-\frac{1+\alpha}{2}}}+ \frac{1}{n\epsilon}\sqrt{d\log(1/\delta)}\big)$} \\
   ~& &  &&\\
 \hline 
\end{tabular}
\end{center}
}}\vspace*{-0.1in}
\end{table*}

\begin{table*}\label{table:summary-pairwise}
{\small{
\caption{Comparison of different $(\epsilon,\delta)$-DP algorithms for pairwise learning.  We report the results for Gradient descent with output perturbation (Output GD), Localized Gradient descent (Localized GD) and SGD with gradient perturbation (Gradient SGD).   }  
\begin{center}
\begin{tabular}{ c|cccccc } 
 \hline
 \multirow{1}{*}{\textbf{Work}}  &   \multirow{1}{*}{ \textbf{Method}}  &   \multirow{1}{*}{ \textbf{Lipschitz}} & \multirow{1 }{*}{\textbf{Smooth}} &  \multirow{1}{*}{\textbf{Low-noise}} & \multirow{1}{*}{\textbf{Gradient complexity}}  & \multirow{1}{*}{\textbf{Utility }}  \\
 \hline
 \multirow{2}{*}{\cite{huai2020pairwise}}  &  \multirow{2}{*}{Output GD}  &  \multirow{2}{*}{$\checkmark$}      &  \multirow{2}{*}{$\checkmark$}&  \multirow{2}{*}{$\times$} &  \multirow{2}{*}{$\O\big( n^2 \big)$} &  \multirow{2}{*}{$\O\big( \frac{1}{\sqrt{n}\epsilon}\sqrt{d\log(1/\delta)} \big)$} \\
 &&&\\
  \hline
 \multirow{2}{*}{\cite{xue2021differentially}} &  \multirow{2}{*}{Localized GD}   &  \multirow{2}{*}{$\checkmark$}         &  \multirow{2}{*}{$\checkmark$}  &  \multirow{2}{*}{$\times$} &  \multirow{2}{*}{$\O\big(n^3\log(1/{\delta})\big)$}  &  \multirow{2}{*}{$\O\big(\frac{1}{\sqrt{n}}+\frac{1}{n\epsilon}\sqrt{d\log(1/\delta)}\big)$}  \\ 
 &&&\\
 \hline
 \multirow{4}{*}{\cite{yang2021simple}}  &  \multirow{2}{*}{Localized SGD}  &  \multirow{2}{*}{$\checkmark$}        &  \multirow{2}{*}{$\checkmark$}&  \multirow{2}{*}{$\times$}&  \multirow{2}{*}{$\O\big(n\log(1/{\delta}) \big)$} &  \multirow{2}{*}{$\O\big(\frac{1}{\sqrt{n}}+\frac{1}{n\epsilon}\sqrt{d}\log^{\frac{3}{2}}(1/\delta)\big)$}  \\ 
   &&&\\
 ~ &  \multirow{1.5}{*}{Localized SGD}  &  \multirow{1.5}{*}{$\checkmark$}          &  \multirow{1.5}{*}{$\times$}&  \multirow{1.5}{*}{$\times$}&  \multirow{1.5}{*}{$\O(n^2\log(1/\delta) )$} &  \multirow{1.5}{*}{$\O\big(\frac{1}{\sqrt{n}}+\frac{1}{n\epsilon}\sqrt{d \log(1/\delta)}\big)$} \\
 &&&\\
   \hline
 \multirow{7}{*}{Ours}  &  \multirow{2}{*}{Gradient SGD} &  \multirow{2}{*}{$\checkmark$}             &  \multirow{2}{*}{$\checkmark$} &  \multirow{2}{*}{$\times$}&  \multirow{2}{*}{$\O(n )$}&  \multirow{2}{*}{$\O\big(\frac{1}{\sqrt{n}}+\frac{1}{n\epsilon}\sqrt{d\log(1/\delta)}\big)$}  \\ 
&&&\\
 ~ &  \multirow{1.5}{*}{Gradient SGD} &  \multirow{1.5}{*}{$\checkmark$}           &  \multirow{1.5}{*}{$\checkmark$}&  \multirow{1.5}{*}{$\checkmark$}&  \multirow{1.5}{*}{$\O(n )$} &  \multirow{1.5}{*}{$\O\big( \frac{1}{n\epsilon}\sqrt{d\log(1/\delta)}\big)$} \\
 ~ &  \multirow{3}{*}{Gradient SGD} &  \multirow{3}{*}{$\checkmark$}          &  \multirow{3}{*}{$\alpha$-H\"older }&  \multirow{3}{*}{$\times$}&  \multirow{3}{*}{$\O\big(n^{
\frac{2-\alpha}{1+\alpha}}+n\big)$}  &  \multirow{3}{*}{$\O\big(\frac{1}{\sqrt{n}}+\frac{1}{n\epsilon}\sqrt{d\log(1/\delta)}\big)$} \\
&&&\\
 ~ &  \multirow{2}{*}{Gradient SGD} &  \multirow{2}{*}{$\checkmark$}            &  \multirow{2}{*}{$\alpha$-H\"older }&  \multirow{2}{*}{$\checkmark$}&  \multirow{2}{*}{$\O\big(n^{
\frac{2}{1+\alpha}} \big)$} &  \multirow{2}{*}{$\O\big( {n^{-\frac{1+\alpha}{2}}}+ \frac{1}{n\epsilon}\sqrt{d\log(1/\delta)}\big)$} \\
&&&\\
 \hline 
\end{tabular}
\end{center}
}} 
\end{table*}

We provide an affirmative answer to the above question under a low-noise condition (also referred as a realizability condition in the literature)\cite{schliserman2022stability,srebro2010smoothness,shamir2021gradient,lei2020fine,nagayasu2022asymptotic}, which assumes that there exists a model within the considered hypothesis space perfectly fits the underlying data distribution. Under this condition, we conduct a comprehensive study of   DP-SGD for both pointwise and pairwise learning as well as both smooth and non-smooth losses, which is able to provide faster utility bounds in terms of the excess population risk. Our main contributions are listed as follows:
\begin{itemize}[itemsep=1ex, leftmargin=6mm] 

\item Firstly, we are concerned with the standard pointwise learning problems where the loss function $f(\cdot;z)$   on a single datum $z = (x,y)$. For this case, we show that DP-SGD with gradient perturbation algorithm can achieve the   rate $\O\big(\frac{1}{\sqrt{n}}+ \frac{1}{n\epsilon}\sqrt{d\log(1/\delta)} \big)$  for both strongly smooth and $\alpha$-H\"older smooth losses, which match the  results in the recently work \cite{wang2022differentially}. Under a low-noise condition,  we remove the term $\O\big(\frac{1}{\sqrt{n}}\big)$ and achieve the excess risk bound of the order $\O\big(\frac{1}{n\epsilon} \sqrt{d\log(1/\delta)}  \big)$ for strongly smooth losses.  Further, 
a better excess risk rate $\O\big(n^{-\frac{1+\alpha}{2}}+ \frac{1}{n\epsilon}\sqrt{d\log(1/\delta)}  \big)$ is  established for $\alpha$-H\"older smooth losses.

\item Secondly, we study the pairwise learning setting where the loss $f(\cdot;z,z')$ involves a pair of examples $(z,z').$ In this learning setting, we propose a simple differentially private SGD algorithm for pairwise learning with utility guarantees. Specifically,  for strongly smooth losses, our algorithm only requires gradient complexity $\O( n )$ to achieve the optimal excess risk rate,  while   \cite{xue2021differentially} and \cite{yang2021simple} require $\O(n^3\log(1/\delta))$ and $\O(n\log(1/\delta))$, respectively. We also show that this rate can  be achieved even if the loss is non-smooth. Further, for both strongly smooth and non-smooth pairwise losses, we establish faster excess risk bounds under a low-noise condition. 
To the best of our knowledge, this is the first utility analysis which provides the excess risk bounds better than $\O\big(\frac{1}{\sqrt{n}}+ \frac{1}{n\epsilon}\sqrt{d\log(1/\delta)} \big)$ for privacy-preserving pairwise learning.  

\end{itemize}

\subsection{Related Work}
In this subsection, we review the relevant work on DP-SGD which are close to our work. We discuss them in the pointwise and pairwise learning settings, respectively.  

For pointwise learning, \cite{bassily2019private} established the excess population risk bounds in the order of $\O\big(\frac{1}{\sqrt{n}}+\frac{1}{n\epsilon}\sqrt{d\log(1/\delta)}\big)$ for $(\epsilon,\delta)$-differentially private stochastic convex optimization algorithms for both strongly smooth and non-smooth losses, which match the lower bound given in \cite{bassily2014private}. However, their algorithms have a large gradient complexity (measured by the total number of computing the gradient). Specifically, their analysis establishes gradient complexity $\O(n^{1.5} \sqrt{\epsilon} + (n\epsilon)^{2.5} ( d \log(1/{\delta}))^{-1} )$ and $\O( n^{4.5}\sqrt{\epsilon}+(n)^{6.5}\epsilon^{4.5} (d\log(1/\delta))^{-2} )$ for strongly smooth and non-smooth losses, respectively.  
\cite{feldman2020private} proposed a private phased SGD algorithm for strongly smooth losses, which can achieve the optimal excess risk rate with a linear gradient complexity $\O(n)$. The work \cite{bassily2020stability} developed a DP-SGD algorithm with gradient perturbation which improved the gradient complexity to $\O(n^2)$ for non-smooth losses. The work most related to our paper is \cite{wang2022differentially}, which studied DP-SGD with gradient perturbation. They established the optimal excess risk bounds for strongly smooth and $\alpha$-H\"older smooth losses with gradient complexity $\O(n)$ and $\O( n^{\frac{2-\alpha}{1+\alpha}} + n)$, respectively, which recover the results in \cite{feldman2020private} and \cite{bassily2020stability}.   However, they didn't obtain the fast rates in  the low-noise case which is the main focus of our paper. For clarity, 
we  list in Table 1 the comparison of our work again other existing work in terms of utility (excess risk) bounds, assumptions on loss function and the gradient complexity of DP-SGD in the pointwise learning setting. 

For pairwise learning, \cite{huai2020pairwise} studied private gradient descent (GD) with output perturbation and proved that the proposed algorithm can achieve the excess risk rate $\O(\frac{1}{\sqrt{n}\epsilon}\sqrt{d\log(1/\delta)})$ for Lipschitz and strongly smooth losses. \cite{xue2021differentially} proposed a private localized GD algorithm, which can achieve the optimal excess risk rate with gradient complexity $\O(n^3\log(1/\delta))$ for Lipschitz and strongly smooth losses. The work \cite{yang2021simple} developed a DP-SGD algorithm with an iterative localization technique and derived the (nearly) optimal excess risk bounds for strongly smooth and non-smooth losses with gradient complexity $\O( n\log(1/\delta) )$ and $\O(n^2\log(1/\delta))$, respectively. In this work, we are interested in DP-SGD for both strongly smooth and $\alpha$-H\"older smooth losses as well as the low-noise case. Table 2 summarizes the comparison of our work against the existing methods in terms of the utility (excess risk) bounds, assumptions on losses and the gradient of DP-SGD in the pairwise learning setting.

\bigskip

\noindent {\bf Organization of the paper. } The remaining parts of the paper are organized as follows. In Section~\ref{sec:probelm-formulation}, we present the formulations of pointwise and pairwise learning together with basic concepts of differential privacy. In Sections~\ref{sec:main-result}, we introduce the DP-SGD algorithms in the settings of pointwise learning and pairwise learning and present our main results. The main proofs are given in Section \ref{sec:proof}.   Section~\ref{sec:conclu} concludes the paper.

\section{Learning Setting and Preliminaries}\label{sec:probelm-formulation}
 
Let $\rho$ be a probability measure defined on $\Z=\X\times \Y$, where $\X\subset\R^d$ is an input space and $\Y\subset\R$ is an output space. In the standard framework of statistical learning theory \cite{bousquet2002stability,vapnik1999nature},  one
considers the problem of learning from a training dataset $S=\{ z_i\}_{i=1}^n$, where $z_i$ is independently drawn from $\rho$.  In the subsequent subsections, we describe the settings of pointwise and pairwise learning, the definition of differential privacy, and illustrate the goal of utility analysis.

\subsection{Pointwise and Pairwise Learning}
In the task of pointwise learning such as classification and regression, we aim to learn a model $ \bw\in\W\subset\R^d$ from training data $S$ and measure the quality of $\bw$ using a pointwise loss function $f(\bw;z)$   on a single datum $z = (x,y).$ The expected population risk for pointwise learning is given by $ F(\bw) = \E_{z\sim \rho}[f(\bw;z)]$. The corresponding empirical risk minimization (ERM) problem based on training dataset $S$ is defined by
\begin{equation}\label{eq:erm-point}
    \min_{\bw\in\W}\Big\{ F_S(\bw)=\frac{1}{n} \sum_{i=1}^n f(\bw;z_i) \Big\}.
\end{equation} 
In contrast to pointwise learning,  the performance of a model $\bw$ for pairwise learning is measured on a pair of examples $(z,z')$ by a loss function $f(\bw;z,z')$~\cite{wang2022error,yang2021simple,lei2020sharper,lei2021generalization}. 
Many machine learning problems can be formulated as learning with pairwise loss functions including AUC maximization \cite{cortes2003auc,gao2013one,liu2018fast,ying2016stochastic,zhao2011online}, metric learning \cite{bellet2013survey,cao2016generalization,jin2009regularized},  a minimum error entropy  principle \cite{hu2015regularization} and ranking \cite{agarwal2009generalization,clemenccon2008ranking}.
we use $\bar{F}(\bw)$ to denote the population risk, i.e., $\bar{F}(\bw) = \E_{z,z'\sim \rho}[f(\bw;z,z')]. $   Let $\bw^{*}=\arg\min_{\bw\in\W} \bar{F}(\bw)$ be the best model, and let $[n]:=\{1,\ldots,n\}$.  
The  ERM problem on training data $S$  is given by 
\begin{equation}\label{eq:erm-pair}
    \min_{\bw\in\W} \Big\{ \bar{F}_S(\bw)=\frac{1}{n(n-1)} \sum_{i,j\in[n], i \ne j} f(\bw;z_i,z_j)  \Big\}.
\end{equation}

\subsection{Definition and Property of Differential Privacy}

As a privacy-preserving technology with a rigorous mathematical guarantee, DP has been widely used in several areas \cite{GONG2020484,gong2020preserving,LI202214,yang2021simple}. 
Its definition is stated formally as follows. 
\begin{definition}[Differential Privacy (DP)\cite{dwork2006calibrating}]\label{def:DP}
We say a randomized algorithm $\A$ satisfies $(\gep, \delta)$-DP if, for any two neighboring datasets $S$ and $S'$  differing at one data point and any event $E$ in the output space of $\A$, there holds 
\[
\mathbb{P}(\A (S)\in E) \le e^\gep \mathbb{P}(\A(S')\in E)+ \gd.
\]
In particular,  we call it satisfies $\gep$-DP if $\gd=0$. 
\end{definition} 
To show a randomized algorithm satisfies DP, we need the following concept called  $\ell_2$-sensitivity. Let $\|\cdot\|_2$ denote the Euclidean norm. 
\begin{definition}\label{def:sensitivity}
The $\ell_2$-sensitivity of a function (mechanism) $\mathcal{M}:\Z^n \rightarrow \W$ is defined as 
$
\gD  = \sup_{S, S'} \|\mathcal{M}(S) - \mathcal{M}(S')\|_2,
$ where $S$ and $S'$ are neighboring datasets differing at one data point. 
\end{definition}
A basic mechanism to achieve $(\epsilon,\delta)$-DP is called Gaussian mechanism, which is shown as follows.  
 \begin{lemma}[\cite{dwork2014algorithmic}]\label{lem:gaussian-noise}
Given a  function $\mathcal{M}: \Z^n \rightarrow \W$ with the $\ell_2$-sensitivity $\gD $ and a dataset $S\subset \Z^n$, and assume that $\sigma \geq \frac{\sqrt{2\log(1.25/\delta)}\Delta }{\epsilon}$. The following Gaussian mechanism yields $(\gep, \gd)$-DP: 
\[\label{eq:gausian-mech}
\mathcal{G}(S,\sigma) := \mathcal{M}(S) + \bb, ~~\bb \sim \mathcal{N}(0, \sigma^2\mathbf{I}_d), 
\]
where $\mathbf{I}_d$ is the identity matrix in $\mathbb{R}^{d\times d}$.  
\end{lemma}

We are interested in DP-SGD with strongly smooth and  $\alpha$-H\"older smooth losses, respectively. 
\begin{definition} 
We say a function $\bw \rightarrow f(\bw)$ is $L$-strongly smooth with $L>0$ if, for any  $\bw,\bw'\in \W$, there holds $ f(\bw)\le f(\bw' ) + \langle \partial f(\bw' ), \bw-\bw' \rangle + \frac{L}{2} \|\bw-\bw'\|_2^2 $, where $\partial f(\cdot)$ denotes a (sub)gradient of $f$.
We  say a function  $\bw \rightarrow f(\bw)$ is $\alpha$-H\"older smooth with   $\alpha\in[0,1)$ and parameter $L$   if for any $\bw,\bw'\in\W$, there holds
$\|\partial f(\bw ) -\partial f(\bw' )\|_2\le L\|\bw-\bw'\|_2^\alpha.  $
\end{definition}
The smoothness parameter $\alpha\in[0,1)$ characterizes the smoothness of the  function $f$. Specifically, if $\alpha=0$, then $f $   is Lipschitz continuous as considered in Definition~\ref{def:lip-convex} below. This definition  instantiates many non-smooth loss functions including the  hinge loss $ \max\big\{0, \big(1- y\bw^{\top}\bx  \big)^q\big\}$ for $q$-norm soft margin SVM  and the $q$-norm loss $|y-\bw^{\top}\bx|^q$ in regression with $q\in [1,2]$.

\subsection{Target of Utility Analysis}
We move on to describing the target of utility analysis of a randomized algorithm $\A$ to solve the ERM problems  \eqref{eq:erm-point} or \eqref{eq:erm-pair}. For simplicity, we elaborate this by taking   pointwise learning as example and the same procedure can apply to the case of pairwise learning. 

To this end, let $\A(S)$ denote  the output of $\A$ based on the training dataset $S$ for pointwise learning. The utility of the output of a randomized algorithm is measured by the \textit{excess population risk}  $F(\A(S))-F(\bw^{*})$, where 
 $\bw^{*}=\arg\min_{\bw\in\W} F(\bw)$ is the one with the best prediction performance over $\W$.   To examine the excess population risk, we use the following error decomposition:
\begin{equation}\label{eq:error-decom}
    \E_{S,\A}[F(\A(S))-F(\bw^{*})]= \E_{S,\A}[F(\A(S)) - F_S(\A(S))] + \E_{S,\A}[F_S(\A(S)) - F_S(\bw^{*})], 
\end{equation}
where $\E_{S,\A}[\cdot]$ denotes the expectation w.r.t. both the randomness of $S$ and the internal randomness of $\A$. The first term $\E_{S,\A}[F(\A(S)) - F_S(\A(S))]$ is called the generalization error, which measures the discrepancy between the expected risk and the empirical one. It can be handled by the stability analysis \cite{bassily2020stability,bousquet2002stability,hardt2016train,kuzborskij2018,lei2020fine}. The  second term is called the optimization error. We will use  tools in optimization theory to control this term.

Throughout the paper, we assume the loss function $f$ is convex and Lipschitz continuous with respect to  (w.r.t.)  the first argument.

\begin{definition}\label{def:lip-convex}
We say a function $\bw \rightarrow f(\bw)$ is convex  if, for any  $\bw,\bw'\in\W$, there holds 
 $f(\bw )\ge f(\bw' ) + \langle \partial f(\bw' ), \bw -\bw' \rangle$. We say a function $\bw \rightarrow f(\bw)$ is $G$-Lipschitz continuous with $G>0$ if, for any  $\bw,\bw'\in\W$, there holds $|f(\bw )-f(\bw' )|\le G\|\bw-\bw'\|_2 $. 
\end{definition}

\section{Main Results}\label{sec:main-result}
We present our main results in this section. First, we propose the differentially private SGD algorithm for pointwise learning, and  systematically study the privacy and utility guarantees of the proposed algorithm. Then, we turn to  pairwise learning problems. We present a simple differentially private SGD algorithm for pairwise learning and provide its  privacy and utility guarantees.

\subsection{DP-SGD for Pointwise Learning}

In this subsection, we are interested in differentially private SGD for pointwise learning. 
To achieve $(\epsilon,\delta)$-differential privacy, we resort to the gradient perturbation mechanism,  i.e., adding Gaussian noise to the stochastic gradient. The detailed algorithm is described in Algorithm~\ref{alg1}. 
In particular, in each iteration $t$, the algorithm randomly selects a sample $z_{i_t}$ according to the uniformly distribution over $[n]$,  and then updates the model parameter $\bw_{t+1}$ based on the noising gradient $\partial f(\bw_t;z_{i_t})+\bb_t$ with $\bb_t\sim \N(0,\sigma^2\mathbf{I}_d)$. After $T$ iterations, Algorithm~\ref{alg1} outputs the private average model ${\bw}_\priv = \frac{1}{T}\sum_{t=1}^T\bw_t$, whose  privacy guarantee is established in the following algorithm.    

\begin{algorithm}[t]
\begin{algorithmic}[1]
\caption{DP-SGD for pointwise learning}\label{alg1}
\STATE{\bf Inputs:}  Data $S= \{z_{i}\in \Z: i=1,\ldots, n\}$, loss function $f(\bw;z)$ with Lipschitz parameter $G$, the  convex set $\W\subseteq \R^d$, step size $\{\eta_t\}$, privacy parameters $\gep$, $\gd$, and constant $\beta$.
\STATE{{\bf Set:}  $\bw_1=\0$ }
\FOR { $t=1$ to $T$ } 
\STATE{Sample $i_t\sim \text{Unif}([n])$} 
\STATE{$\bw_{t+1}=\proj_{\W}\big(\bw_t-\eta_t (\partial f(\bw_t;z_{i_t})+\bb_t)\big)$, where $\bb_t\sim \N(0,\sigma^2 \mathbf{I}_d)$ with  $\sigma^2=\frac{14 G^2 T }{\beta n^2 \epsilon} \Big( \frac{\log(1/\delta)}{(1-\beta)\epsilon} +1\Big)$}
\ENDFOR
\STATE  {\bf return:}   ${\bw}_\priv = \frac{1}{T}\sum_{t=1}^T\bw_t$ 
\end{algorithmic}
\end{algorithm}

\begin{theorem}[Privacy guarantee]\label{thm:dp-private}
Suppose that the loss function $f$ is convex and $G$-Lipshitz. Then Algorithm~\ref{alg1}  with some $\beta\in(0,1)$  satisfies $(\epsilon,\delta)$-DP if  $\sigma^2\ge 2.68G^2$ and $\lambda-1\le \frac{\sigma^2}{6G^2}\log\Big( \frac{n}{\lambda\big(1+\frac{\sigma^2}{4G^2}\big)} \Big)$ with $\lambda=\frac{\log(1/\delta)}{(1-\beta)\epsilon}+1$.
\end{theorem}
\begin{remark}
In Algorithm~\ref{alg1}, 
 the variance $\sigma^2$ of  the Gaussian noise $\bb_t$ depends on a constant $\beta\in(0,1)$, which should satisfy the conditions $\sigma^2\ge 2.68G^2$ and $\lambda-1\le \frac{\sigma^2}{6G^2}\log\Big( \frac{n}{\lambda\big(1+\frac{\sigma^2}{4G^2}\big)} \Big)$. \cite{wang2022differentially} studied DP-SGD with gradient perturbation for $\alpha$-H\"older smooth losses and gave a sufficient condition for the existence of $\beta$ under a specific parameter setting. Specifically, they proved that if $n>18$, $T=n$ and $\delta=1/n^2$, then there exists at least one  $\beta\in(0,1)$ such that DP-SGD satisfies $(\epsilon,\delta)$-DP when   $  \epsilon \ge    { 7(n^{\frac{1}{3}}-1) + 4\log(n) n +7 }/{( 2n(n^{\frac{1}{3}}-1)) }$. Indeed, our algorithm can be seen as a special case of their algorithm with $\alpha=0$. Hence, we can also show  the existence of $\beta$ under the same setting. 
\end{remark}

Now, we establish the utility guarantee for strongly smooth losses.  Part (a) in the following theorem provides the optimal utility bound  for a general setting, i.e., the “pessimistic” case $F(\bw^{*})>0$.  Part (b) of Theorem~\ref{thm:excess-smooth-point} focuses on the  low-noise setting, i.e., the  optimistic case $F(\bw^{*})=0$,  where the best possible model $\bw^{*}$ can achieve zero error. This setting is particularly intriguing in the context of deep learning, where models may possess more parameters than training examples.  
\begin{theorem}[Utility guarantee  for smooth losses]\label{thm:excess-smooth-point}
Suppose $f$ is nonnegative, convex, $G$-Lipschitz and $L$-smooth. Let $\bw_{\priv}$ be the output by Algorithm~\ref{alg1} with $T$ iterations. Then the following statements hold true.
\begin{enumerate}[label=({\alph*})]
\item  If we choose $\eta_t=c\min\Big\{ \frac{1}{\sqrt{n}}, \frac{\epsilon}{\sqrt{d\log(1/\delta)}} \Big\}\le \min \{ 2/L, 1 \}$ for some constant $c>0$ and  $T\asymp n$, then
\[ \E_{S,\A}[F(\bw_{\priv})] -  F(\bw^{*})  =\O\Big( \frac{1}{ \sqrt{n} } + \frac{  \sqrt{d\log(1/\delta)}}{n\epsilon}\Big). \]

\item If $F(\bw^{*})=0$,  we choose $\eta_t=\frac{c \epsilon}{\sqrt{d\log(1/\delta)}}\le\min \{ 2/L, 1 \}$ for some constant $c>0$  and $T\asymp n$, then
\[ \E_{S,\A}[F(\bw_{\priv})] -  F(\bw^{*})  =\O\Big( \frac{\sqrt{d\log(1/\delta)}}{n\epsilon} \Big). \]
\end{enumerate}
\end{theorem}
\begin{remark}\label{rmk:utility-point}
\cite{wang2022differentially}  established the optimal rate for DP-SGD algorithm and improved the gradient complexity to $\O(n)$ when the loss is strongly smooth and the parameter space is bounded. 
Our bound (part (a) in Theorem~\ref{thm:excess-smooth-point}) can achieve the optimal rate with gradient complexity $\O(n)$ when the loss is  strongly smooth and Lipschitz continuous. Compared with \cite{wang2022differentially}, we need a further Lipschitz continuous assumption. However,  this assumption can be removed when we assume the parameter domain is bounded in our setting. Indeed, the smoothness of $f$ implies  that the upper bound of the gradient can be controlled by the diameter of  parameter domain $R$, i.e.,  $\|\partial f(\bw;z)\|_2\le \sup_{z}\|\partial f(0;z)\|_2+ L\|\bw\|_2\le \sup_{z}\|\partial f(0;z)\|_2+L R$, where $L$ is the smoothness parameter.   Hence, our result can achieve  the optimal rate under the same assumptions as \cite{wang2022differentially}. 
In the optimistic case with $F(\bw^{*})=0$, Part (b) in Theorem~\ref{thm:excess-smooth-point} removes the term $\O\Big(\frac{1}{\sqrt{n}}\Big)$ and further improves the excess population risk rate to $\O\Big( \frac{1}{n\epsilon}{\sqrt{d\log(1/\delta)}}  \Big)$ with gradient complexity $\O(n)$ for  strongly smooth losses under a  low-noise condition.  A very recent work \cite{kang2022sharper}  provided the excess population risk rate $\O\Big( \frac{1}{n\epsilon}{\sqrt{d\log(1/\delta)}}  \Big)$ for the private gradient descent algorithm, while they focused on the non-convex setting and assumed Polyak-Łojasiewicz condition holds.
\end{remark} 
Now, we turn to the more general case, i.e., the loss function is $\alpha$-H\"older smooth with $\alpha\in[0,1)$. The following theorem presents the excess population risk bound for $\alpha$-H\"older smooth losses. 
\begin{theorem}[Utility guarantee for non-smooth losses]\label{thm:excess-nonsmooth-point}
Suppose $f$ is nonnegative, convex, $G$-Lipschitz and $\alpha$-H\"older smooth with parameter $L$ and $\alpha\in[0,1)$. Let $\bw_{\priv}$ be the output of  Algorithm~\ref{alg1} with $T$ iterations. Then the following statements hold true.
\begin{enumerate}[label=({\alph*})]
\item If $\alpha\ge 1/2,$  we choose    $\eta_t = c \min\Big\{\frac{1}{\sqrt{n}}, \frac{\epsilon}{\sqrt{d\log(1/\delta)}}\Big\}\le \min \{ 2/L, 1 \} $ for some constant $c>0$ and $T\asymp n$. If $\alpha<1/2,$ we choose   $\eta_t=c \min\Big\{ n^{\frac{3(\alpha-1)}{2(1+\alpha)}}, \frac{\epsilon}{\sqrt{d\log(1/\delta)}} \Big\}\le \min \{ 2/L, 1 \}$ for some constant $c>0$, and $T\asymp n^{\frac{2-\alpha}{1+\alpha}}$. Then  \begin{align*}
 \E_{S,\A}[  F(\bw_{\priv})] - F(\bw^{*})    = \O\Big( \frac{1}{\sqrt{n}} + \frac{\sqrt{d \log(1/\delta)}}{n\epsilon} \Big).
\end{align*}
\item  If $F(\bw^{*})=0$, we choose $\eta_t = c \min\Big\{n^{\frac{\alpha^2+ 2\alpha-3}{2(1+\alpha)}}, \frac{n\epsilon}{T\sqrt{d\log(1/\delta)}}\Big\}\le \min \{ 2/L,1 \} $ for some constant $c>0$  and $T\asymp n^{\frac{2}{1+\alpha}}$.  Then
\begin{align*}
 \E_{S,\A}[  F(\bw_{\priv})] - F(\bw^{*})    = \O\Big( \frac{1}{n^{\frac{1+\alpha}{2}}} + \frac{\sqrt{d \log(1/\delta)}}{n\epsilon} \Big).
\end{align*}
\end{enumerate}
\end{theorem}
\begin{remark}\label{rmk:utility-point-nonsmooth}
 \cite{wang2022differentially} studied DP-SGD with gradient perturbation for $\alpha$-H\"older smooth losses and showed that the algorithm can achieve the optimal rate $\O\big(\frac{1}{\sqrt{n}}+ \frac{1}{n\epsilon}\sqrt{d\log(1/\delta)} \big)$  with gradient complexity $\O\big( n^{\frac{2-\alpha}{1+\alpha}} +n \big)$. Our result (Part (a)  in Theorem~\ref{thm:excess-nonsmooth-point}) matches their bounds  with the same gradient complexity. As discussed in Remark~\ref{rmk:utility-point}, although we need a further Lipschitz condition, we can also recover their result under the same setting when the parameter domain is bounded. Analogous to the smooth case, Part (b) in Theorem~\ref{thm:excess-nonsmooth-point}  derives the excess population risk bound better than  $\O\big(\frac{1}{\sqrt{n}}+ \frac{1}{n\epsilon}\sqrt{d\log(1/\delta)} \big)$ . 
 To the best of our knowledge, this is the first excess population risk bound of the order $\O\big( n^{-\frac{1+\alpha}{2}} + \frac{1}{n\epsilon}{\sqrt{d\log(1/\delta)}} \big)$ for private SGD with non-smooth losses.
\end{remark}

\subsection{DP-SGD for Pairwsie Learning}

\begin{algorithm}[t]
\begin{algorithmic}[1]
\caption{DP-SGD for pairwise learning (\texttt{DP-SGD-pairwise})}\label{alg2}
\STATE{\bf Inputs:}  Data $S= \{z_{i}\in \Z: i=1,\ldots, n\}$, loss function $f(\bw;z,z')$ with Lipschitz parameter $G$, the  convex set $\W\subseteq \R^d$, step size $\{\eta_t\}$, privacy parameters $\gep$, $\gd$, and constant $\beta$.
\STATE{{\bf Set:}  $\bw_1=\0$ }
\FOR { $t=1$ to $T$ } 
\STATE{Sample $(i_t,j_t) $ uniformly over all pairs $\{ (i,j) : i, j \in [n], i\neq j \}$} 
\STATE{$\bw_{t+1}=\proj_{\W}\big(\bw_t-\eta_t (\partial f(\bw_t;z_{i_t},z_{j_t})+\bb_t)\big)$, where $\bb_t\sim \N(0,\sigma^2 \mathbf{I}_d)$ with  $\sigma^2=\frac{56 G^2 T }{\beta n^2 \epsilon} \Big( \frac{\log(1/\delta)}{(1-\beta)\epsilon} +1\Big)$}
\ENDFOR
\STATE  {\bf return:}   ${\bw}_\priv = \frac{1}{T}\sum_{t=1}^T\bw_t$ 
\end{algorithmic}
\end{algorithm}

In this subsection, we first present the differentially private SGD algorithm for pairswise learning, and then establish its privacy and utility guarantees. 
The proposed algorithm  is described in Algorithm~\ref{alg2}. In particular, in iteration $t$, the algorithm draws a pair $\{(i_t,j_t)\}$ from the uniform distribution over all pairs $\{(i,j):i,j\in[n],i\neq j\}$. Then the parameter is updated by the noised gradient $\partial f(\bw_t;z_{i_t},z_{j_t})+\bb_t$ with $\bb_t\sim \mathcal{N}(0,\sigma^2\mathbf{I}_d)$. 
The following theorem establishes the privacy guarantee for Algorithm~\ref{alg2}. 
\begin{theorem}[Privacy guarantee]\label{thm:dp-pairwise}
Suppose that the loss function $f$ is convex and $G$-Lipschitz. Then Algorithm~\ref{alg2} with some $\beta\in(0,1)$ satisfies $(\epsilon,\delta)$-DP if   $\sigma^2 \ge 2.68G^2$ and $\lambda-1\le \frac{\sigma^2}{6G^2}\log\Big(\frac{n}{2\lambda\big(1+\frac{\sigma^2}{4G^2}\big)}\Big)$ with $\lambda=\frac{\log(1/\delta)}{(1-\beta)\epsilon}+1.$
\end{theorem}

 By combining the stability results  and the optimization error bounds (Lemmas  \ref{lem:stability-pairwise} and  \ref{lem:opt-pairwise} below) together, we establish the following utility guarantees for Algorithm~\ref{alg2} for strongly smooth and non-smooth losses, respectively. 
\begin{theorem}[Utility guarantee for smooth losses]\label{thm:excess-pair}
Suppose $f$ is nonnegative, convex, $G$-Lipschitz and $L$-smooth. Let $\{\bw_t\}$ be produced by Algorithm~\ref{alg2} with $T$ iterations.  Then the following statements hold true. 
\begin{enumerate}[label=({\alph*})]
\item  If we choose $\eta_t=c\min\Big\{ \frac{1}{\sqrt{n}}, \frac{\epsilon}{\sqrt{d\log(1/\delta)}} \Big\}\le \min\{ 2/L,1\}$ for some constant $c>0$ and $T\asymp n$, then
\[ \E_{S,\A}[\bar{F}(\bw_{\priv})] -\bar{F}(\bw^{*})  =\O\Big( \frac{1}{ \sqrt{n} } + \frac{  \sqrt{d\log(1/\delta)}}{n\epsilon}\Big). \]

\item If $\bar{F}(\bw^{*})=0$,  we choose $\eta_t=\frac{c \epsilon}{\sqrt{d\log(1/\delta)}}\le \min\{2/L,1\}$ for some constant $c>0$   and $T\asymp n$, then
\[ \E_{S,\A}[\bar{F}(\bw_{\priv})] -  \bar{F}(\bw^{*})  =\O\Big( \frac{\sqrt{d\log(1/\delta)}}{n\epsilon} \Big). \]
\end{enumerate}
\end{theorem}
 \begin{remark}\label{rmk:utility-pair}
   We now compare our results with the related work for pairwise learning. Under the strongly smooth and Lipschitz continuous assumptions,  \cite{huai2020pairwise} proposed the gradient descent with output perturbation algorithm to achieve DP and provided the excess population risk bound in the order of $\O\big( \frac{1}{ {\sqrt{n}\epsilon}} {\sqrt{d\log(1/\delta)}} \big)$ with gradient complexity $\O\big(n^2\big)$. \cite{xue2021differentially} improved the excess population risk rate to $\O\big( \frac{1}{\sqrt{n}}+\frac{1}{n\epsilon}\sqrt{d\log(1/\delta)} \big)$ by proposing a localized gradient descent algorithm with a large gradient complexity $\O\big( n^3\log(1/\delta) \big)$. \cite{yang2021simple} presented a simple localized DP-SGD algorithm which can achieve the optimal excess risk rate $\O\big( \frac{1}{\sqrt{n}}+\frac{1}{n\epsilon}\sqrt{d\log(1/\delta)} \big)$ up to a $\log(1/\delta)$ term. Their algorithm   needs the gradient complexity   $\O\big(n\log(1/\delta)\big)$. Our result (Part (a) in Theorem~\ref{thm:excess-pair})  shows that 
  our algorithm can achieve the optimal excess risk rate $\O\big( \frac{1}{\sqrt{n}}+\frac{1}{n\epsilon}\sqrt{d\log(1/\delta)} \big)$ only with the gradient complexity $\O(n)$ for strongly smooth losses, which significantly reduces the computational complexity of the algorithm. Under a low-noise condition, Part (b) removes the term $\O\big(\frac{1}{\sqrt{n}}\big)$ and derives  the excess population risk bound of the order $\O\big( \frac{1}{n\epsilon}\sqrt{d\log(1/\delta)} \big)$, which only need the gradient complexity in the order of $\O(n)$. To the best of our knowledge, this is the first  excess population risk bound in the order of  $\O\big( \frac{1}{n\epsilon}\sqrt{d\log(1/\delta)} \big)$  for privacy-preserving pairwise learning.  
 \end{remark}
 The following theorem establishes the utility bounds for Algorithm~\ref{alg2} when the loss is non-smooth. 
\begin{theorem}[Utility guarantee for non-smooth losses]\label{thm:excess-nonsmooth-pair}
Suppose $f$ is nonnegative, convex, $G$-Lipschitz and $\alpha$-H\"older smooth with parameter $L$ and $\alpha\in[0,1)$. Let $\{\bw_t\}$ be produced by Algorithm~\ref{alg2} with $T$ iterations.  Then the following statements hold true. 
\begin{enumerate}[label=({\alph*})]
\item If $\alpha\ge 1/2,$  we choose     $\eta_t = c \min\Big\{\frac{1}{\sqrt{n}}, \frac{\epsilon}{\sqrt{d\log(1/\delta)}}\Big\}\le \min\{2/L,1\} $ for some constant $c>0$ and $T\asymp n$.  If $\alpha<1/2,$ we choose   $\eta_t=c \min\Big\{ n^{\frac{3(\alpha-1)}{2(1+\alpha)}}, \frac{\epsilon}{\sqrt{d\log(1/\delta)}} \Big\}\le \min\{2/L, 1\}$ for some constant $c>0$, and $T\asymp n^{\frac{2-\alpha}{1+\alpha}}$. Then  \begin{align*}
 \E_{S,\A}[  \bar{F}(\bw_{\priv})] -\bar{F}(\bw^{*})    = \O\Big( \frac{1}{\sqrt{n}} + \frac{\sqrt{d \log(1/\delta)}}{n\epsilon} \Big).
\end{align*}
\item  If $\bar{F}(\bw^{*})=0$, we choose $\eta_t = c \min\Big\{n^{\frac{\alpha^2+ 2\alpha-3}{2(1+\alpha)}}, \frac{n\epsilon}{T\sqrt{d\log(1/\delta)}}\Big\}\le \min\{2/L,1\} $ for some constant $c>0$  and $T\asymp n^{\frac{2}{1+\alpha}}$.  Then
\begin{align*}
 \E_{S,\A}[  \bar{F}(\bw_{\priv})] - \bar{F}(\bw^{*})    = \O\Big( \frac{1}{n^{\frac{1+\alpha}{2}}} + \frac{\sqrt{d \log(1/\delta)}}{n\epsilon} \Big).
\end{align*}
\end{enumerate}
\end{theorem}
\begin{remark}
Part (a) in the above theorem shows that the optimal rate $\O\big(\frac{1}{\sqrt{n}} + \frac{1}{n\epsilon}\sqrt{d\log(1/\delta)} \big)$ can be achieved with the same gradient complexity $T\asymp n $ if $\alpha\ge 1/2$. For the case $\alpha<1/2$, the same rate can be also achieved with a larger gradient complexity $\O\big(n^{\frac{2-\alpha}{1+\alpha}}\big)$.    
For non-smooth losses (i.e., $\alpha=0$), \cite{yang2021simple} established  the optimal excess population risk rate  for localized DP-SGD algorithm with gradient complexity $\O\big(n^2\log(1/\delta )\big)$ for Lipschitz continuity losses. Under the same assumptions, Part (a)   with $\alpha=0$ implies  that the optimal rate can be achieved with gradient complexity $\O( n^2 )$. Our result reduces the computational cost by a factor of $\O\big(\log(1/\delta)\big)$ in this case.   Part (b) establishes the first excess population risk bounds better than $\O\big(\frac{1}{\sqrt{n}} + \frac{1}{n\epsilon}\sqrt{d\log(1/\delta)} \big)$ in the case with low-noise for privacy-preserving pairwise learning. 
\end{remark}

\section{Proofs of Main Results}\label{sec:proof}
Before presenting the detailed proof, we first introduce  some definitions and useful lemmas. 
To establish   tighter privacy analysis of DP-SGD, we introduce the definition of R\'{e}nyi differential privacy (RDP) which provides   tighter composition and amplification results for   iterative algorithms.   
\begin{definition}[RDP \cite{mironov2017renyi}]\label{def:RDP}
For $\lambda > 1$, $\rho > 0$, a randomized mechanism $\A$ satisfies $(\lambda, \rho)$-RDP, if,  for all neighboring datasets $S$ and $S'$, we have 
     $$ D_{\lambda}\big(\A(S)\parallel \A(S')\big):= \frac{1}{\lambda-1}\log \int  \Big( \frac{ P_{\A(S)}(\theta) }{ P_{\A(S')}(\theta) }  \Big)^\lambda    d P_{\A(S')}(\theta) \le \rho,$$
    where $P_{\A(S)}(\theta)$ and $P_{\A(S')}(\theta)$ are the density of $\A(S) $ and $\A(S')$, respectively. 
\end{definition}

 The following lemma shows the privacy amplification of RDP by uniform subsampling, which is fundamental to establish privacy guarantees of noisy SGD algorithms. 
\begin{lemma}[\cite{liang2020exploring}]\label{lem:uniform-rdp} Consider a function $\mathcal{M}:  \Z^n\rightarrow \W$ with the $\ell_2$-sensitivity $\Delta$,  and a dataset $S\subset\Z^n$.  
{The Gaussian mechanism $\mathcal{G}(S,\sigma)=\mathcal{M}(S)+\bb$, where $\bb\sim \mathcal{N}(0,\sigma^2\mathbf{I}_d)$, } applied to  a subset of samples that are drawn uniformly without replacement with subsampling rate $p$ satisfies 
$(\lambda,3.5p^2\lambda \Delta^2/\sigma^2)$-RDP  if $\sigma^2\geq 0.67 \Delta^2$ and $\lambda-1 \leq \frac{2\sigma^2}{3\Delta^2} \log \big(\frac{1}{\lambda p (1+ \sigma^2/\Delta^2)} \big)$. 
\end{lemma}

 We say a sequence of mechanisms $(\A_1,\ldots,\A_k)$ are chosen adaptively if $\A_i$ can be chosen based on the outputs of the previous mechanisms  $\A_1(S),\ldots,\A_{i-1}(S)$ for any $i\in[k]$.
\begin{lemma}[Adaptive Composition of RDP \cite{mironov2017renyi}]\label{lem:composition_RDP}
If a mechanism $\A$ consists of a sequence of  adaptive mechanisms $(\A_1,\ldots,\A_k)$ with $\A_i$ satisfying  $(\lambda, \rho_i)$-RDP, $i\in[k]$, then $\A$ satisfies $(\lambda, \sum_{i=1}^k \rho_i)$-RDP. 
\end{lemma}
 
 The relationship between RDP and $(\epsilon,\delta)$-DP is given as follows. 
\begin{lemma}[From RDP to $(\epsilon,\delta)$-DP \cite{mironov2017renyi}]\label{lemma:RDP_to_DP}
	If a randomized mechanism $\mathcal{A}$ satisfies $(\lambda,\rho)$-RDP, then $\mathcal{A}$ satisfies $(\rho+\log(1/\delta)/(\lambda-1),\delta)$-DP for all $\delta\in(0,1)$.
\end{lemma}
A fundamental property of DP called post-processing property is introduced as follows. It implies that a differentially private output can be arbitrarily transformed by using some data-independent functions.  
\begin{lemma}[Post-processing \cite{mironov2017renyi}]\label{lemma:post-processing}
Let  $\A: \Z^n \rightarrow \W_1 $  satisfy $(\lambda, \rho)$-RDP  and $f: \W_1 \rightarrow \W_2$ be an arbitrary function. Then $f \circ \A : \Z^n \rightarrow \W_2$ satisfies $(\lambda, \rho)$-RDP.  
\end{lemma}

Let   $M=\sup_{z\in\Z}f(0;z)$.  Define
\begin{equation}\label{eq:c_1}
  c_{\ga,1}=\begin{cases}
     (1+1/\ga)^{\frac{\ga}{1+\ga}}L^{\frac{1}{1+\ga}}, & \mbox{if } \ga> 0, \\
                 M+L, & \mbox{if } \ga=0.
               \end{cases}
\end{equation} 
Our analysis requires to use a self-bounding property \cite{srebro2010smoothness,ying2017unregularized} for strongly smooth and $\alpha$-H\"older smooth losses, which means that gradients can be controlled by function values. 
\begin{lemma}[Self-bounding property]\label{lem:self-bounding} Suppose $f$ is nonnegative.  If $f$ is $L$-strongly smooth, then there holds $ \|\partial f(\bw;z)\|_2 \le \sqrt{ 2L f(\bw;z)}   $ \  for any  $ \bw\in\R^d, z\in \Z$. If $f$ is  $\alpha$-H\"older smooth with $L>0$ and $\alpha\in[0,1)$, then for $c_{\alpha,1}$ defined in \eqref{eq:c_1} we have
$ \|\partial f(\bw;z)\|_2 \le c_{\alpha,1} f^{\frac{\alpha}{1+\alpha}}(\bw;z) $  for any  $ \bw\in\R^d, z\in \Z.$ 
\end{lemma}
 
We will use the following concept of on-average argument stability to study the generalization error. 
\begin{definition}[On-average argument stability~\cite{lei2020fine}]\label{def:avg-stability}
Let $S=\{z_1,\ldots,z_n\}$ and $S'=\{z_1',\ldots,z_n'\}$ be drawn independently from $\rho$. For any $i\in[n]$, denote $S^{(i)}=\{ z_1,\ldots,z_{i-1},z_i',z_{i+1},\ldots,z_n \}$ as the set from $S$ by replacing the $i$-th element with $z_i'$. We say an algorithm $\A$ is on-average argument $\epsilon$-stable if 
\[ \E_{S,S',\A}\Big[ \frac{1}{n}\sum_{i=1}^n \|\A(S)-\A(S^{(i)})\|^2_2 \Big]\le \epsilon. \]
\end{definition}

\subsection{Proofs for Pointwise Learning}\label{appendix-pointwise}
We first give the proof of the privacy guarantee for Algorithm~\ref{alg1}. Specifically, according to the Lipschitz continuity of $f$, we can show that the $\ell_2$-sensitivity of $\mathcal{M}_t=\partial f(\bw_t;z_{i_t})$ is $2G$. 
Then by Lemma~\ref{lem:gaussian-noise} and the post-processing property, we know that $\bw_{t+1}$ is $\Big(\frac{\log(1/\delta)}{(1-\beta)\epsilon}+1,\frac{\beta\epsilon}{T}\Big)$-RDP for any $t=1,\ldots, T$.  
Further, we  use the adaptive composition theorem (Lemma~
\ref{lem:composition_RDP}) and the connection between RDP and DP (Lemma~\ref{lemma:RDP_to_DP})  to show that $\bw_{\priv}$ satisfies $(\epsilon,\delta)$-DP. 
The detailed proof is shown as follows. 
\begin{proof}[Proof of Theorem~\ref{thm:dp-private}]
 
For each iteration $t$, 
let $\mathcal{A}_t=\mathcal{M}_t+\bb_t$, where $\mathcal{M}_t=\partial f(\bw_t;z_{i_t})$. For any $\bw_t\in\W$ and any $z_{i_t}, z'_{i_t} \in \Z$, the Lipschitz continuity of $f$ implies
\[ \|\partial f(\bw_t;z_{i_t}) - \partial f(\bw_t;z'_{i_t})\|_2\le  \|\partial f(\bw_t;z_{i_t})\|_2 + \|\partial f(\bw_t;z'_{i_t})\|_2\le 2G. \]
From the definition of sensitivity (see Definition~\ref{def:sensitivity}), we know the $\ell_2$-sensitivity of $\mathcal{M}_t$ is bounded by $2G$. Note that
\[ \sigma^2=\frac{14G^2T}{\beta n^2 \epsilon}\Big( \frac{\log(1/\delta)}{(1-\beta)\epsilon}+1 \Big). \]
According to Lemma~\ref{lem:uniform-rdp} with $p=1/n$, we know $\mathcal{\A}_t$ is $\Big(\lambda, \frac{\lambda \beta \epsilon}{T\big( \frac{\log(1/\delta)}{(1-\beta)\epsilon} +1\big)} \Big)$-RDP as long as $\sigma^2\ge 2.68G^2$ and $\lambda-1\le \frac{\sigma^2}{6G^2}\log\Big( \frac{n}{\lambda\big(1+\frac{\sigma^2}{4G^2}\big)} \Big)$ hold. 

Let $\lambda=\frac{\log(1/\delta)}{(1-\beta)\epsilon}+1$, then we get $\A_t$ is $\Big(\frac{\log(1/\delta)}{(1-\beta)\epsilon}+1,\frac{\beta\epsilon}{T}\Big)$-RDP. Further, Lemma~\ref{lemma:post-processing} implies that $\bw_{t+1}$ is $\Big(\frac{\log(1/\delta)}{(1-\beta)\epsilon}+1,\frac{\beta\epsilon}{T}\Big)$-RDP for any $t=1,\ldots, T$. According to the adaptive composition theorem of RDP (see Lemma~\ref{lem:composition_RDP}), we know  Algorithm~\ref{alg1} is $\Big(\frac{\log(1/\delta)}{(1-\beta)\epsilon}+1, \beta\epsilon \Big)$-RDP. Finally, the relationship between RDP and DP (Lemma~\ref{lemma:RDP_to_DP}) implies that  Algorithm~\ref{alg1} is $(\epsilon,\delta)$-DP if $\sigma^2\ge 2.68G^2$ and $\lambda-1\le \frac{\sigma^2}{6G^2}\log\Big( \frac{n}{\lambda\big(1+\frac{\sigma^2}{4G^2}\big)} \Big)$ hold. The proof is completed. 
\end{proof}

\bigskip

To study the utility guarantee of Algorithm~\ref{alg1}, we need to estimate the generalization error $\E_{S,\A}[F(\bw_{\priv})-F_S(\bw_{\priv})]$ and the optimization error $\E_{S,\A}[F_S(\bw_{\priv})-F(\bw^{*})]$, respectively. We will use on-average argument stability to study the generalization
error, which measures the sensitivity of the output model of an algorithm.  
The relationship  between  generalization error and  on-average argument stability is established in the following lemma  \cite{lei2020fine}. 

\begin{lemma}[Generalization via on-average stability] \label{lem:gen}
Let $\A$ be on-average $\nu$-stable. Let $\gamma>0$. 
\begin{enumerate}[label=({\alph*})]
\item   If $f$ is nonnegative and $L$-smooth, then
\[ \E_{S,\A}[F(\A(S))-F_S(\A(S))]\le \frac{L}{\gamma}\E_{S,\A}[ F_S(\A(S)) ] + \frac{(L+\gamma)\nu}{2 }.
\]
\item   If $f$ is nonnegative, convex and $\alpha$-H\"older smooth with parameter $L$ and $\alpha\in[0,1)$, then  
\[ \E_{S,\A}[F(\A(S))-F_S(\A(S))]\le  \frac{c^2_{\alpha,1}}{2\gamma} \E_{S,\A}[ F^{\frac{2\alpha}{1+\alpha}}(\A(S)) ] + \frac{\gamma \nu}{2 } . \]
\end{enumerate}
\end{lemma}

Since the noise added to the gradient in each iteration is the same for the neighboring datasets, then the noise addition does not impact the stability analysis. Therefore,  the on-average argument stability of non-private SGD equals that of private SGD. 
We can use the following lemma directly to give  the stability bounds of Algorithm~\ref{alg1} for both  strongly smooth and non-smooth losses \cite{lei2020fine}. 
\begin{lemma}[On-average stability bounds]\label{lem:stability-point}
Suppose $f$ is nonnegative and convex. Let $S,S' $ and $S^{(i)}$ be constructed as Definition~\ref{def:avg-stability}. Let $\{\bw_{t}\}$ and $\{\bw_t^{(i)}\}$ be produced by Algorithm~\ref{alg1}   based on $S$ and $S^{(i)}$, respectively.
\begin{enumerate}[label=({\alph*})]
    \item If $f$ is $L$-smooth and  $\eta_t\le 2/L$ for all $t\in[T]$, then
    \[  \E_{S,S',\A}\Big[\frac{1}{n}\sum_{i=1}^n\| \bw_{t+1}-\bw_{t+1}^{(i)} \|_2^2\Big]\le \frac{8e(1+t/n)L}{n} \sum_{j=1}^t  \eta_j^2 \E_{S,\A}[ F_S(\bw_j) ].  \]
\item If $f$ is $\alpha$-H\"older smooth with parameter $L$ and $\alpha\in[0,1)$, then
\begin{align*}
    \E_{S,S',\A}\Big[\frac{1}{n}\sum_{i=1}^n\| \bw_{t+1}-\bw_{t+1}^{(i)} \|_2^2\Big]\le c_{\alpha,3}^2e \sum_{j=1}^t \eta_j^{\frac{2}{1-\alpha}} + \frac{4ec_{\alpha,1}^2(1+t/n) }{n}\sum_{j=1}^t \eta_j^2\E_{S,\A}\Big[ F_S^{\frac{2\alpha}{1+\alpha}} (\bw_j)\Big]  ,
\end{align*}
where $c_{\alpha,3}= \sqrt{ \frac{1-\alpha}{1+\alpha}}(2^{-\alpha}L)^{\frac{1}{1-\alpha}}$.  
\end{enumerate}
\end{lemma}
The following theorem presents generalization bounds of DP-SGD for both smooth and non-smooth losses, which directly follows from Lemma~\ref{lem:gen} and Lemma~\ref{lem:stability-point}. 
\begin{theorem}[Generalization bounds]\label{thm:generalization-point}
Suppose $f$ is nonnegative and convex. Let $\W=\R^d$ and let $\A$ be Algorithm~\ref{alg1} with $T$ iterations. Let $\gamma>0$. 
\begin{enumerate}[label=({\alph*})]
    \item If $f$ is $L$-smooth and  $\eta_t\le 2/L$ for all $t\in[T]$, then
    \[\E_{S,\A}[F(\bw_{\priv})-F_S(\bw_{\priv})]\le \frac{L}{\gamma}\E_{S,\A}[ F_S(\bw_{\priv}) ] +  \frac{4e(L+\gamma)(1+t/n)L}{n} \sum_{t=1}^T  \eta_t^2 \E_{S,\A}[ F_S(\bw_t) ].\]
    \item If $f$ is $\alpha$-H\"older smooth with parameter $L$ and $\alpha\in[0,1)$, then
    \begin{align*}
    &\E_{S,\A}[F(\bw_{\priv})-F_S(\bw_{\priv})]\\
    &\le \frac{c^2_{\alpha,1}}{2\gamma} \E_{S,\A}[ F^{\frac{2\alpha}{1+\alpha}}(\bw_{\priv}) ] + \frac{\gamma  }{2 }\Big(c_{\alpha,3}^2e \sum_{t=1}^T \eta_t^{\frac{2}{1-\alpha}} + \frac{4ec_{\alpha,1}^2(1+t/n) }{n}\sum_{t=1}^T \eta_j^2\E_{S,\A}\Big[ F_S^{\frac{2\alpha}{1+\alpha}} (\bw_t)\Big] \Big). 
    \end{align*}  
\end{enumerate}
\end{theorem}
In the following theorem, we use techniques in optimization theory to control the optimization error in expectation. Recall $\bw^{*}=\arg\min_{\bw\in\W} F(\bw) $.  Let 
\begin{equation}\label{eq:c_2}
  c_{\ga,2}=\begin{cases}
     \frac{1-\alpha}{1+\alpha}\big(2\alpha /(1+\alpha)\big)^{\frac{2\alpha}{1-\alpha}}c_{\alpha,1}^{\frac{2+2\alpha}{1-\alpha}}, & \mbox{if } \ga> 0 \\
                  c_{\alpha,1}^2, & \mbox{if } \ga=0.
               \end{cases}
\end{equation} 
\begin{theorem}[Optimization error]\label{thm:opt-point}
Suppose $f$ is nonnegative and convex. Let $\{\bw_t\}$ be produced by Algorithm~\ref{alg1}. Assume the step size $\eta_t$ is nonincreasing.
\begin{enumerate}[label=({\alph*})]  
    \item If $f$ is $L$-smooth, then
    \begin{align*}
          \sum_{j=1}^t \eta_j \E_{\A}[F_S(\bw_j )\!-\!  F_S(\bw^{*} )] 
    \le &\Big(\frac{1}{2}\!+\!3 L  \eta_1\Big)\|  \bw^{*} \|_2^2\!+\!3L       
     \sum_{j=1}^t \big(3\eta_j^3\sigma^2d + 2\eta_j^2  F_S(\bw^{*} )    \big)\!+\!\sum_{j=1}^t  3\eta_j^2\sigma^2 d.
    \end{align*}  
\item If $f$ is   $\alpha$-H\"older smooth with parameter $L$ and $\alpha\in[0,1)$, 
\begin{multline*}
     \sum_{j=1}^t \eta_j \E_{\A}[F_S(\bw_j ) - F_S(\bw^{*} )] \le \frac{1}{2} \|  \bw^{*} \|_2^2  \\ + \frac{3}{4}c_{\alpha,1}^2\Big(\sum_{j=1}^t \eta_j^2\Big)^{\frac{1-\alpha}{1+\alpha}} \Big[2\eta_1  \| \bw^{*} \|_2^2     +   
     \sum_{j=1}^t \big(6\eta_j^3\sigma^2d + 4\eta_j^2  F_S(\bw^{*} )   + 3c_{\alpha,2}\eta_j^{\frac{3-\alpha}{1-\alpha}}  \big) \Big]^{\frac{2\alpha}{1+\alpha}}  +  \sum_{j=1}^t  3\eta_j^2\sigma^2 d. 
\end{multline*}  
\end{enumerate}

\end{theorem}

\begin{proof} 
Note the projection operator $\proj$ is non-expansive. Then for any $\alpha\in[0,1]$, we have
\begin{align}
    &\|\bw_{t+1} - \bw^{*}\|_2^2
    \le \|\bw_t - \eta_{t}(\partial f(\bw_t;z_{i_t}) + \bb_t) -\bw^{*}\|_2^2\nonumber\\
    &=\| \bw_t  -\bw^{*} \|_2^2 + \eta_t^2\| \partial f(\bw_t;z_{i_t}) + \bb_t \|_2^2 +2\eta_t\langle \bw^{*} -\bw_t, \partial f(\bw_t;z_{i_t})+\bb_t \rangle\nonumber\\
    &\le \| \bw_t  -\bw^{*} \|_2^2 + \frac{3}{2}\eta_t^2\| \partial f (\bw_t;z_{i_t})\|_2^2 +3\eta_t^2\| \bb_t \|_2^2 +2\eta_t\langle \bw^{*} -\bw_t, \partial f(\bw_t;z_{i_t})+\bb_t \rangle \nonumber\\
    &\le \| \bw_t  -\bw^{*} \|_2^2 + \frac{3}{2}c_{\alpha,1}^2\eta_t^2f^{\frac{2\alpha}{1+\alpha}}(\bw_t;z_{i_t}) + 3\eta_t^2\| \bb_t \|_2^2 + 2\eta_t\big( f(\bw^{*};z_{i_t}) - f(\bw_t;z_{i_t}) \big) + 2\eta_t\langle \bw^{*} -\bw_t, \bb_t \rangle,\label{eq:opt-1}
\end{align}
where in the second inequality we used $(a+b)^2\le (1+p)a^2+(1+1/p)b^2$ with $p=1/2$, and the last inequality is due to the self-bounding property (Lemma~\ref{lem:self-bounding}) and the convexity of $f$. 

Rearranging the above inequality, we get
\begin{align*}
    &2\eta_t[f(\bw_t;z_{i_t}) - f(\bw^{*};z_{i_t})]\\ & \le  \| \bw_t  -\bw^{*} \|_2^2 - \|\bw_{t+1} - \bw^{*}\|_2^2 + \frac{3}{2}c_{\alpha,1}^2\eta_t^2f^{\frac{2\alpha}{1+\alpha}}(\bw_t;z_{i_t}) + 3\eta_t^2\| \bb_t \|_2^2   + 2\eta_t\langle \bw^{*} -\bw_t, \bb_t \rangle. 
\end{align*}
Taking a summation over $j$ and noting $\bw_1=\mathbf{0}$, we know
\begin{align*}
    &2\sum_{j=1}^t \eta_j[f(\bw_j;z_{i_j}) - f(\bw^{*};z_{i_j})]\\& \le   \|  \bw^{*} \|_2^2  + \frac{3}{2}c_{\alpha,1}^2\sum_{j=1}^t\eta_j^2f^{\frac{2\alpha}{1+\alpha}}(\bw_j;z_{i_j}) +  \sum_{j=1}^t \big(3\eta_j^2\| \bb_j \|_2^2   + 2\eta_j\langle \bw^{*} -\bw_j, \bb_j \rangle\big). 
\end{align*}
Note that $\bw_j$ is independent of $i_j$, we can take an expectation w.r.t. $\A$  and get
\begin{align}\label{eq:opt-4}
     \sum_{j=1}^t \eta_j \E_{\A}[F_S(\bw_j ) - F_S(\bw^{*} )]&=\sum_{j=1}^t \eta_j \E_{\A}[f  (\bw_j;z_{i_j}) - f(\bw^{*};z_{i_j})]\nonumber\\
    &\le \frac{1}{2} \|  \bw^{*} \|_2^2  + \frac{3}{4}c_{\alpha,1}^2\sum_{j=1}^t\eta_j^2\E_\A[f^{\frac{2\alpha}{1+\alpha}}(\bw_j;z_{i_j} )] +  \sum_{j=1}^t  3\eta_j^2\sigma^2 d   ,
\end{align}
where we used $\E_\A[\|\bb_j\|_2^2]=\sigma^2d$ and $\E_\A[\langle \bw^{*} -\bw_j, \bb_j \rangle]=0$ since $\bb_j$ is a Gaussian vector with mean $0$ and variance $\sigma^2$, and $\bw^{*}-\bw_j$ is independent of $\bb_j$. 

To control the right hand side of \eqref{eq:opt-4}, we have to estimate $\sum_{j=1}^t\eta_j^2\E_\A[f^{\frac{2\alpha}{1+\alpha}}(\bw_j;z_{i_j} )] $. By Young's inequality $ab\le p^{-1}|a|^p + q^{-1}|b|^q$ with $a,b\in\R$ and $p^{-1}+q^{-1}=1$, for any $t\in[T]$ we have
\begin{align*}
    \eta_t c_{\alpha,1}^2 f^{\frac{2\alpha}{1+\alpha}}(\bw_t;z_{i_t})&=\Big(\frac{1+\alpha}{2\alpha} f(\bw_t;z_{i_t}) \Big)^{\frac{2\alpha}{1+\alpha}} \Big( \frac{2\alpha}{1+\alpha} \Big)^{\frac{2\alpha}{1+\alpha}}c_{\alpha,1}^2 \eta_t\\
    &\le \frac{2\alpha}{1+\alpha} \Big(\frac{1+\alpha}{2\alpha} f(\bw_t;z_{i_t}) \Big)^{\frac{2\alpha}{1+\alpha} \frac{1+\alpha}{2\alpha} } + \frac{1-\alpha}{1+\alpha} \Big( \big( \frac{2\alpha}{1+\alpha} \big)^{\frac{2\alpha}{1+\alpha}}c_{\alpha,1}^2 \eta_t\Big)^{\frac{1+\alpha}{1-\alpha}}\\
    &= f( \bw_t;z_{i_t})  + c_{\alpha,2}\eta_t^{\frac{1+\alpha}{1-\alpha}}.
\end{align*} 
Putting the above inequality back into \eqref{eq:opt-1} yields
\begin{align*}
    \| \bw_{t+1}\!-\!\bw^{*} \|_2^2&\le \| \bw_t\!-\!\bw^{*} \|_2^2   + 3\eta_t^2\| \bb_t \|_2^2 + 2\eta_t  f(\bw^{*};z_{i_t}) - \frac{1}{2}\eta_t f(\bw_t;z_{i_t})  + \frac{3}{2}c_{\alpha,2}\eta_t^{\frac{2}{1-\alpha}} + 2\eta_t\langle \bw^{*} -\bw_t, \bb_t \rangle. 
\end{align*}
Rearranging the above inequality and multiplying both sides by $\eta_t$, we get
\begin{align*}
    &\eta_t^2 f(\bw_t;z_{i_t})\\
    &\le 2\eta_t \big(\| \bw_t  -\bw^{*} \|_2^2   -\| \bw_{t+1}  -\bw^{*} \|_2^2     \big) +   6\eta_t^3\| \bb_t \|_2^2 + 4\eta^2_t  f(\bw^{*};z_{i_t})   + 3c_{\alpha,2}\eta_t^{\frac{3-\alpha}{1-\alpha}} + 4\eta^2_t\langle \bw^{*} -\bw_t, \bb_t \rangle\\
    &\le 2\eta_t  \| \bw_t  -\bw^{*} \|_2^2   - 2\eta_{t+1}\| \bw_{t+1}  -\bw^{*} \|_2^2     +   
    6\eta_t^3\| \bb_t \|_2^2 + 4\eta^2_t  f(\bw^{*};z_{i_t})   + 3c_{\alpha,2}\eta_t^{\frac{3-\alpha}{1-\alpha}} + 4\eta^2_t\langle \bw^{*} -\bw_t, \bb_t \rangle,
\end{align*}
where we assume $\eta_t\ge \eta_{t+1}$ for all $t\in [T-1]$. 

 Taking a summation over $j$ and noting $\bw_1=\mathbf{0}$, we know 
\begin{align}\label{eq:opt-2}
    \sum_{j=1}^t\eta_j^2 f(\bw_j;z_{i_t})  
    &\le 2\eta_1  \| \bw^{*} \|_2^2     +   
     \sum_{j=1}^t \big(6\eta_j^3\| \bb_j \|_2^2 + 4\eta^2_j  f(\bw^{*};z_{i_j})   + 3c_{\alpha,2}\eta_j^{\frac{3-\alpha}{1-\alpha}} + 4\eta^2_j\langle \bw^{*} -\bw_j, \bb_j\rangle\big).
\end{align}
Note $x \mapsto x^{\frac{2\alpha}{1+\alpha}}$ is concave. Then Jensen's inequality implies
\begin{align}\label{eq:opt-3}
    &\sum_{j=1}^t \eta_j^2 f^{\frac{2\alpha}{1+\alpha}}(\bw_j;z_{i_j})  \le  \sum_{j=1}^t \eta_j^2 \Bigg( \frac{  \sum_{j=1}^t \eta_j^2f(\bw_j;z_{i_j}) }{\sum_{j=1}^t \eta_j^2}  \Bigg)^{\frac{2\alpha}{1+\alpha}}
    = \Big(\sum_{j=1}^t \eta_j^2\Big)^{\frac{1-\alpha}{1+\alpha}} \Big[ \sum_{j=1}^t \eta_j^2f(\bw_j;z_{i_j}) \Big]^{\frac{2\alpha}{1+\alpha}}\nonumber\\
    &\le \Big(\sum_{j=1}^t \eta_j^2\Big)^{\frac{1-\alpha}{1+\alpha}} \Big[2\eta_1  \| \bw^{*} \|_2^2     +   
     \sum_{j=1}^t \big(6\eta_j^3\| \bb_j \|_2^2 + 4\eta^2_j  f(\bw^{*};z_{i_j})   + 3c_{\alpha,2}\eta_j^{\frac{3-\alpha}{1-\alpha}} + 4\eta^2_j\langle \bw^{*} -\bw_j, \bb_j\rangle\big) \Big]^{\frac{2\alpha}{1+\alpha}}.
\end{align} 
Plugging the above inequality back into \eqref{eq:opt-4},  we have
\begin{align*} 
     &\sum_{j=1}^t \eta_j \E_{\A}[F_S(\bw_j ) - F_S(\bw^{*} )]
    \le \frac{1}{2} \|  \bw^{*} \|_2^2  + \sum_{j=1}^t  3\eta_j^2\sigma^2 d \\ & + \frac{3}{4}c_{\alpha,1}^2\Big(\sum_{j=1}^t \eta_j^2\Big)^{\frac{1-\alpha}{1+\alpha}} \E_\A \Big[2\eta_1  \| \bw^{*} \|_2^2\!+\!\sum_{j=1}^t \big(6\eta_j^3\| \bb_j \|_2^2 \!+\!4\eta^2_j  f(\bw^{*};z_{i_j})\!+\! 3c_{\alpha,2}\eta_j^{\frac{3-\alpha}{1-\alpha}}\!+\! 4\eta^2_j\langle \bw^{*} -\bw_j, \bb_j\rangle\big) \Big]^{\frac{2\alpha}{1+\alpha}}\\
    &\le \frac{1}{2} \|  \bw^{*} \|_2^2  + \frac{3}{4}c_{\alpha,1}^2\Big(\sum_{j=1}^t \eta_j^2\Big)^{\frac{1-\alpha}{1+\alpha}} \Big[2\eta_1  \| \bw^{*} \|_2^2\!+\!\sum_{j=1}^t \big(6\eta_j^3\sigma^2d\!+\! 4\eta_j^2  F_S(\bw^{*} )\!+\! 3c_{\alpha,2}\eta_j^{\frac{3-\alpha}{1-\alpha}}  \big) \Big]^{\frac{2\alpha}{1+\alpha}} +  \sum_{j=1}^t  3\eta_j^2\sigma^2 d, 
\end{align*}
where the last inequality used Jensen's inequality for concave mapping and $\E_\A[\langle \bw^{*} -\bw_j, \bb_j\rangle]=0$.   Part (b) is proved. From the definition we know that $\alpha$-H\"older smoothness with  $\alpha=1$ corresponds to the strongly smoothness of $f$. Hence, Part (a) in the theorem directly follows by setting $\alpha=1$ in the above inequality. 
\end{proof}
  
Now, we can establish the proofs of the excess population risk bounds of DP-SGD for pointwise learning by combining Theorem~\ref{thm:generalization-point} and Theorem~\ref{thm:opt-point} together.  First, we give the proof for the strongly smooth case (i.e., Theorem~\ref{thm:excess-smooth-point}). 
\begin{proof}[Proof of Theorem~\ref{thm:excess-smooth-point}]
Putting stability bounds for smooth losses (Part (a) in Lemma~\ref{lem:stability-point}) back into Part (a) of Lemma~\ref{lem:gen}, we   get
\[ \E_{S,\A}[F(\bw_{t+1}) ] \le \Big(1+\frac{L}{\gamma}\Big) \E_{S,\A}[ F_S(\bw_{t+1}) ] + \frac{4e(L+\gamma)(1+t/n)L}{ n} \sum_{j=1}^t  \eta_j^2 \E_{S,\A}[ F_S(\bw_j) ].
  \]
Note that $\bw_j$ is independent of $\bb_j$ and $i_j$. Eq.\eqref{eq:opt-2} implies   
\begin{align*}
     &\sum_{j=1}^t\eta_j^2 \E_{S,\A}[F_S(\bw_j) ]  =\sum_{j=1}^t\eta_j^2 \E_{S,\A}[f(\bw_j;z_{i_t}) ] \\&\le 2\eta_1  \| \bw^{*} \|_2^2     +   
     \sum_{j=1}^t \big(6\eta_j^3 \E_\A[\| \bb_j \|_2^2] + 4\eta^2_j  \E_{S,\A}[f(\bw^{*};z_{i_j})]   + 4\eta^2_j \E_\A[\langle \bw^{*} -\bw_j, \bb_j\rangle]\big) \\
    &\le 2\eta_1  \| \bw^{*} \|_2^2     +   
     \sum_{j=1}^t \big(6\eta_j^3\sigma^2 d + 4\eta^2_j \ F(\bw^{*})   \big),
\end{align*}
where we used $\E_\A[\| \bb_j \|_2^2]=\sigma^2 d$,  $\E_{S,\A}[f(\bw^{*};z_{i_j})] =  F(\bw^*)$ and $\E_\A[\langle \bw^{*} -\bw_j, \bb_j\rangle]=0$. 

Combining the above two inequalities together, we  get
\begin{align*}
    \E_{S,\A}[F(\bw_{t+1 }) ] \le& \Big(1+\frac{L}{\gamma}\Big) \E_{S,\A}[ F_S(\bw_{t+1 }) ]\\
    &+ \frac{8e(L+\gamma)(1+t/n)L}{ n} \Big[  \eta_1  \| \bw^{*} \|_2^2     +   
     \sum_{j=1}^t \big(3\eta_j^3\sigma^2 d + 2\eta^2_j \ F(\bw^{*})   \big) \Big].
\end{align*}
Multiplying both sides by $\eta_{t+1}$ followed with a summation gives
\begin{align}\label{eq:excess-1}
    \sum_{t=1}^T \eta_t\E_{S,\A}[F(\bw_{t }) ] \le& \Big(1+\frac{L}{\gamma}\Big) \sum_{t=1}^T \eta_t \E_{S,\A}[ F_S(\bw_{t }) ]\nonumber \\
    &+ \frac{8e(L+\gamma)(1+T/n)L}{ n} \sum_{t=1}^T \eta_t \Big[  \eta_1  \| \bw^{*} \|_2^2     +   
     \sum_{j=1}^t \big(3\eta_j^3\sigma^2 d + 2\eta^2_j \ F(\bw^{*})   \big) \Big].
\end{align}
Part (a) in Theorem~\ref{thm:opt-point} implies
\[    \sum_{t=1}^T \eta_t \E_{\A}[F_S(\bw_t ) ]  \le \sum_{t=1}^T \eta_t F_S(\bw^{*} ) + \Big(\frac{1}{2}   +   3L  \eta_1 \Big) \| \bw^{*} \|_2^2     +   
      3 \sum_{t=1}^T \big(3L\eta_t+1\big)\eta_t^2\sigma^2 d  + 4 \sum_{t=1}^T \eta_t^2  F_S(\bw^{*} )    .  \]
Plugging the above inequality back into \eqref{eq:excess-1} and noting $\E_{S }[F_S(\bw^{*})]=F(\bw^{*})$, we get 
\begin{align*}
 &\sum_{t=1}^T \eta_t\E_{S,\A}[F(\bw_{t }) ]\\ &\le \Big(1+\frac{L}{\gamma}\Big)  \Big(\sum_{t=1}^T \eta_t F(\bw^{*} ) + \Big(\frac{1}{2}   +   3L  \eta_1 \Big) \| \bw^{*} \|_2^2     +   
    3\sum_{j=1}^t \big(3L\eta_j+1\big)\eta_j^2\sigma^2 d  + 4 \sum_{j=1}^t \eta_j^2  F(\bw^{*} )  \Big)  \nonumber \\
    &\quad + \frac{8e(L+\gamma)(1+T/n)L}{ n} \sum_{t=1}^T \eta_t \Big[  \eta_1  \| \bw^{*} \|_2^2     +   
     \sum_{j=1}^t \big(3\eta_j^3\sigma^2 d + 2\eta^2_j \ F(\bw^{*})   \big) \Big]. 
\end{align*}
Let $\eta_t=\eta\le \min\{2/ L,1\}$ and assume $T\ge n$. Note $\bw_{\priv}=\frac{1}{T}\sum_{t=1}^T \bw_t$.  Then according to Jensen's inequality, there holds
\begin{align*}
  &\E_{S,\A}[F(\bw_{\priv }) - F(\bw^{*}) ] \\
  &=\O\Bigg(   \Big( \frac{(1+ \gamma^{-1})}{T\eta} + \frac{(1+\gamma) T\eta}{n^2} \Big)\|\bw^{*}\|_2^2  + \Big(  \gamma^{-1} + \big(1 + {\gamma}^{-1}\big)\eta + \frac{(1+\gamma)T^2\eta^2}{n^2} \Big)F(\bw^{*}) \\
  &\quad  + \big( 1+ {\gamma}^{-1}\big)\sigma^2 d \eta + \frac{ (1+\gamma) T^2 \eta^3 \sigma^2 d }{n^2} \Bigg). 
\end{align*}
Recaling that $\sigma^2 d =  \frac{14 G^2 T d }{\beta n^2 \epsilon} \Big( \frac{\log(1/\delta)}{(1-\beta)\epsilon} +1\Big)$,   we further have 
\begin{align}\label{eq:excess-2}
  &\E_{S,\A}[F(\bw_{\priv }) - F(\bw^{*}) ]\nonumber \\
  &=\O\Bigg(   
\big( \frac{(1+\gamma^{-1})}{T\eta} + \frac{(1+\gamma)T\eta}{n^2} \Big)\|\bw^{*}\|_2^2+ \Big( \gamma^{-1} + \frac{ T^2  \eta^2(1+\gamma) }{n^2} + \big(\gamma^{-1}+1\big)\eta \Big)F(\bw^{*}) \nonumber\\
&\quad + \Big(\big(1+\gamma^{-1}\big)\eta + \frac{ T^2\eta^3(1+\gamma) }{n^2}\Big)\frac{Td\log(1/\delta)}{n^2\epsilon^2}\Bigg).  
\end{align}
 
\noindent (a) If we set $T\asymp n$ and $\gamma=\sqrt{n}$, then Eq.\eqref{eq:excess-2} implies
\begin{align*}
  \E_{S,\A}[F(\bw_{\priv })\!-\! F(\bw^{*}) ] =&\O\Bigg( \Big(\frac{1}{\sqrt{n}}\!+\!\eta^2\sqrt{n}\!+\!\eta\Big) F(\bw^{*}) + \big(\frac{1}{n \eta}+ \frac{\eta}{\sqrt{n}} \big) \|\bw^{*}\|_2^2 + \big(\eta + \eta^3\sqrt{n}\big) \frac{ d\log(1/\delta)\eta}{n\epsilon^2}\Bigg).  
\end{align*}
Further let $\eta_t=c/\max\Big\{ \sqrt{n}, \frac{\sqrt{d\log(1/\delta)}}{\epsilon} \Big\}\le \min\{ 2/L, 1\}$ for some constant $c>0$,  then there holds
\begin{align*}
  \E_{S,\A}[F(\bw_{\priv }) - F(\bw^{*}) ] =&\O\Big(  \frac{1}{\sqrt{n}}  + \frac{\sqrt{d\log(1/\delta)}}{n\epsilon}\Big),  
\end{align*}
where we assume $\sqrt{d\log(1/\delta)}=\O(n\epsilon)$ (otherwise the bound will not converge). 

\noindent (b) Consider the low noise case, i.e, $F(\bw^{*})=0$. Let $\gamma = 1$ and $T\asymp n$, then
\begin{align*}
  \E_{S,\A}[F(\bw_{\priv }) - F(\bw^{*}) ] =\O\Big(  
  \big(\frac{1}{n\eta} +\frac{\eta}{n}\big) \|\bw^{*}\|_2^2 +  \frac{ d\log(1/\delta)\eta}{n \epsilon^2}\Big). 
\end{align*}
Let $\eta_t=\frac{c \epsilon}{\sqrt{d\log(1/\delta)}}\le \min\{ 2/L,1\}$ for some constant $c>0$, then
\begin{align*}
  \E_{S,\A}[F(\bw_{\priv }) - F(\bw^{*}) ] =\O\Big(     \frac{\sqrt{ d\log(1/\delta)}}{n \epsilon }\Big). 
\end{align*}
The proof of the theorem is completed. 
\end{proof}
Finally, we provide the proof of utility guarantee for Algorithm~\ref{alg1} when the loss is non-smooth. 
\begin{proof}[Proof of Theorem~\ref{thm:excess-nonsmooth-point}]
Note $\E_S[F_S(\bw^{*})]=F(\bw^{*})$ and  ${\bw}_\priv = \frac{1}{T}\sum_{t=1}^T\bw_t$.  By Jensen's inequality we know
\begin{align}\label{eq:excess-nonsmooth-1}
\E_{S,\A}[F(\bw_{\priv})] &- F(\bw^{*})  =
    \big( \sum_{t=1}^T \eta_t\big)^{-1} \sum_{t=1}^T \eta_t   \E_{S,\A}[F(\bw_t) - F(\bw^{*})]\nonumber\\
    &= \big( \sum_{t=1}^T \eta_t\big)^{-1} \sum_{t=1}^T \eta_t \E_{S,\A}[ F(\bw_t) - F_S(\bw_t) ] 
     + \big( \sum_{t=1}^T \eta_t\big)^{-1} \sum_{t=1}^T \eta_t \E_{S,\A}[  F_S(\bw_t) - F(\bw^{*}) ].
\end{align}
We first estimate the   term $\big( \sum_{t=1}^T \eta_t\big)^{-1} \sum_{t=1}^T \eta_t \E_{S,\A}[ F(\bw_t) - F_S(\bw_t) ]$. Putting part (b) in Lemma~\ref{lem:stability-point} back into part (b) of Lemma~\ref{lem:gen}, we get
\begin{align*}
    &\E_{S,\A}[ F(\bw_{t+1}) - F_S(\bw_{t+1}) ] \\
    &\le \frac{c^2_{\alpha,1}}{2\gamma} \E_{S,\A}[ F^{\frac{2\alpha}{1+\alpha}} (\bw_{t+1}) ] + \frac{e c_{\alpha,3}^2 \gamma}{2 } \sum_{j=1}^t \eta_j^{\frac{2}{1-\alpha}} + \frac{2ec_{\alpha,1}^2\gamma(1+t/n) }{n}\sum_{j=1}^t \eta_j^2\E_{S,\A}\Big[ F_S^{\frac{2\alpha}{1+\alpha}} (\bw_j)\Big].
\end{align*}
Let $\delta_j=\max\big\{\E_{S,\A}[ F(\bw_{j}) ] - \E_{S,\A}[ F_S(\bw_{j})], 0  \big\}$. 
Due to the concavity of $x\mapsto x^{\frac{2\alpha}{1+\alpha}}$, there holds
\begin{align*}
    \E_{S,\A}[ F^{\frac{2\alpha}{1+\alpha}} (\bw_{t+1})] &\le \big( \E_{S,\A}[ F(\bw_{t+1}) ] - \E_{S,\A}[ F_S(\bw_{t+1})] +  \E_{S,\A}[ F_S(\bw_{t+1})] \big)^{\frac{2\alpha}{1+\alpha}}\\
    &\le \delta_{t+1}^{\frac{2\alpha}{1+\alpha}} +  \big(\E_{S,\A}[ F_S(\bw_{t+1})]\big) ^{\frac{2\alpha}{1+\alpha}}.
\end{align*}
Combining the above two inequalities together yields 
\begin{align*}
    \delta_{t+1}\!\le\! \frac{c^2_{\alpha,1}}{2\gamma}\Big(  \delta_{t+1}^{\frac{2\alpha}{1+\alpha}}\!+\!\big(\E_{S,\A}[ F_S(\bw_{t+1})]\big) ^{\frac{2\alpha}{1+\alpha}} \Big)\!+\!\frac{e c_{\alpha,3}^2 \gamma}{2 } \!\sum_{j=1}^t \eta_j^{\frac{2}{1-\alpha}}\!+\!\frac{2ec_{\alpha,1}^2\gamma(1+t/n) }{n}\!\sum_{j=1}^t \eta_j^2\big(\E_{S,\A}[ F_S (\bw_j)]\big)^{\frac{2\alpha}{1+\alpha}}.
\end{align*}
Solving the above inequality of $\delta_{t+1}$  we get 
\begin{align*}
    \delta_{t+1}=\O\Big( \gamma^{\frac{1+\alpha }{\alpha-1}}+ \gamma^{-1} \big(\E_{S,\A}[ F_S(\bw_{t+1})]\big) ^{\frac{2\alpha}{1+\alpha}}  + \gamma \sum_{j=1}^t \eta_j^{\frac{2}{1-\alpha}} + \gamma\big(n^{-1} + T n^{-2}\big) \sum_{j=1}^t \eta_j^2 \big(\E_{S,\A}[ F_S (\bw_j)]\big)^{\frac{2\alpha}{1+\alpha}}  \Big).
\end{align*}
Assuming $T\ge n$, from the definition of $\delta_{t+1}$ we have 
\begin{align*}
    &\big( \sum_{t=1}^T \eta_t\big)^{-1} \sum_{t=1}^T \eta_t \E_{S,\A}[ F(\bw_t) - F_S(\bw_t) ]\\
    &=\O\Big( \gamma^{\frac{1+\alpha }{\alpha-1}} + \gamma \sum_{t=1}^T \eta_t^{\frac{2}{1-\alpha}} + \gamma^{-1} \big( \sum_{t=1}^T \eta_t\big)^{-1} \sum_{t=1}^T \eta_t\big(\E_{S,\A}[ F_S(\bw_{t})]\big) ^{\frac{2\alpha}{1+\alpha}}  + \gamma T n^{-2} \sum_{t=1}^T \eta_t^2 \big(\E_{S,\A}[ F_S (\bw_t)]\big)^{\frac{2\alpha}{1+\alpha}} \Big).
\end{align*}
If we set $\eta_t=\eta$, then there holds
\begin{align}\label{eq:excess-nonsmooth-2}
      &\big( \sum_{t=1}^T \eta_t\big)^{-1} \sum_{t=1}^T \eta_t \E_{S,\A}[ F(\bw_t) - F_S(\bw_t) ]\nonumber\\
    &=\O\Big( \gamma^{\frac{1+\alpha }{\alpha-1}} + \gamma T \eta^{\frac{2}{1-\alpha}} +  \big(\gamma  T  \eta \big)^{-1}\sum_{t=1}^T\eta  \big(\E_{S,\A}[ F_S(\bw_{t+1})]\big) ^{\frac{2\alpha}{1+\alpha}}  + \gamma T n^{-2} \sum_{t=1}^T \eta^2 \big(\E_{S,\A}[ F_S (\bw_t)]\big)^{\frac{2\alpha}{1+\alpha}} \Big)
\end{align}
Since $\E_{\A}[ \langle \bw^{*} - \bw_t,\bb_t \rangle ]=0$, Eq.\eqref{eq:opt-2} with $\eta_t=\eta$ implies 
\begin{align*}
    \sum_{t=1}^T \eta^2  \big(\E_{S,\A}[ F_S (\bw_t)]\big)^{\frac{2\alpha}{1+\alpha}} &\le  \sum_{t=1}^T \eta^2   \Big(\frac{ \sum_{t=1}^T \eta^2   \E_{S,\A}[ F_S (\bw_t)] }{ \sum_{t=1}^T \eta^2   }   \Big)^{\frac{2\alpha}{1+\alpha}} = \Big(\sum_{t=1}^T \eta^2  \Big)^{\frac{1-\alpha}{1+\alpha}} \Big(   \sum_{t=1}^T \eta^2   \E_{S,\A}[ F_S (\bw_t)]   \Big)^{\frac{2\alpha}{1+\alpha}} \\
    &\le \big(T \eta^2  \big)^{\frac{1-\alpha}{1+\alpha}} \Big(   2\eta  \| \bw^{*} \|_2^2     +   
      6T\eta^3\sigma^2 d + 4T\eta^2 F(\bw^{*} )   + 3c_{\alpha,2}T\eta^{\frac{3-\alpha}{1-\alpha}}  \Big)^{\frac{2\alpha}{1+\alpha}}\\
      &=\O\Big( \big(T \eta^2  \big)^{\frac{1-\alpha}{1+\alpha}} \Big(  \eta      +   
      T\eta^3\sigma^2 d +  T\eta^2 F(\bw^{*} )   +  T\eta^{\frac{3-\alpha}{1-\alpha}}   \Big)^{\frac{2\alpha}{1+\alpha}} \Big).
\end{align*}
Dividing both sides by $\eta$, we get
\begin{align*}
    \sum_{t=1}^T \eta \big(\E_{S,\A}[ F_S (\bw_t)]\big)^{\frac{2\alpha}{1+\alpha}} =\O\Big(  T^{\frac{1-\alpha}{1+\alpha}} \eta^{\frac{1-3\alpha}{1+\alpha}} \Big(  \eta      +   
      T\eta^3\sigma^2 d +  T \eta^2 F(\bw^{*} )   +  T\eta^{\frac{3-\alpha}{1-\alpha}}  \Big)^{\frac{2\alpha}{1+\alpha}} \Big).
\end{align*}
Now, plugging the above two inequalities back into \eqref{eq:excess-nonsmooth-2}, we have
\begin{align} \label{eq:excess-nonsmooth-3}
      &\big( \sum_{t=1}^T \eta_t\big)^{-1} \sum_{t=1}^T \eta_t \E_{S,\A}[ F(\bw_t) - F_S(\bw_t) ]\nonumber\\
    &=\O\Bigg( \gamma^{\frac{1+\alpha }{\alpha-1}} + \gamma T \eta^{\frac{2}{1-\alpha}} +  \big(\gamma  T  \eta \big)^{-1}T^{\frac{1-\alpha}{1+\alpha}} \eta^{\frac{1-3\alpha}{1+\alpha}} \Big(  \eta      +   
      T\eta^3\sigma^2 d +  \eta^2 F(\bw^{*} )   +  T\eta^{\frac{3-\alpha}{1-\alpha}}   \Big)^{\frac{2\alpha}{1+\alpha}}
      \nonumber\\
      &\qquad + \gamma Tn^{-2} \big(T \eta^2  \big)^{\frac{1-\alpha}{1+\alpha}} \Big(  \eta      +   
      T\eta^3\sigma^2 d +  T\eta^2 F(\bw^{*} )   +  T\eta^{\frac{3-\alpha}{1-\alpha}}   \Big)^{\frac{2\alpha}{1+\alpha}}\nonumber\\
      &=\O\Bigg( \gamma^{\frac{1+\alpha }{\alpha-1}} + \gamma T \eta^{\frac{2}{1-\alpha}} +  \Big[\gamma^{-1}T^{\frac{-2\alpha}{1+\alpha}} \eta^{\frac{ -4\alpha}{1+\alpha}} + \gamma   n^{-2}   T^{\frac{2}{1+\alpha}} \eta^{\frac{2-2\alpha}{1+\alpha}}\Big] \Big(  \eta      +   
      T\eta^3\sigma^2 d +T  \eta^2 F(\bw^{*} )   +  T\eta^{\frac{3-\alpha}{1-\alpha}}  \Big)^{\frac{2\alpha}{1+\alpha}} \Bigg).
\end{align}
Part (b) in Theorem~\ref{thm:opt-point} with $\eta_t=\eta$ implies
\begin{align} \label{eq:excess-nonsmooth-4}
     &\big( \sum_{t=1}^T \eta \big)^{-1} \sum_{t=1}^T \eta  \E_{S,\A}[F_S(\bw_t ) - F\bw^{*} )]= \big( \sum_{t=1}^T \eta \big)^{-1} \sum_{t=1}^T \eta  \E_{S,\A}[F_S(\bw_t ) - F_S(\bw^{*} )]\nonumber\\ &=\O\Bigg( \frac{1}{T\eta} +   T^{\frac{-2\alpha }{1+\alpha}} \eta^{\frac{1-3\alpha}{1+\alpha}} \Big( \eta    +   
     T \eta^3\sigma^2d + T\eta^2  F(\bw^{*} )   +  T\eta^{\frac{3-\alpha}{1-\alpha}}  \big) \Big)^{\frac{2\alpha}{1+\alpha}}  + \eta \sigma^2 d\Bigg). 
\end{align}
Plugging \eqref{eq:excess-nonsmooth-3} and \eqref{eq:excess-nonsmooth-4} back into \eqref{eq:excess-nonsmooth-1} yields
\begin{align}\label{eq:excess-nonsmooth-5}
 &\E_{S,\A}[  F(\bw_{\priv})] - F(\bw^{*})  \nonumber\\
 &= \O\Bigg(   \Big(\gamma^{-1}T^{\frac{-2\alpha}{1+\alpha}} \eta^{\frac{ -4\alpha}{1+\alpha}} + \gamma   n^{-2}   T^{\frac{2}{1+\alpha}} \eta^{\frac{2-2\alpha}{1+\alpha}} + T^{\frac{-2\alpha }{1+\alpha}} \eta^{\frac{1-3\alpha}{1+\alpha}}\Big) \Big(  \eta      +   
      T\eta^3\sigma^2 d +  T\eta^2 F(\bw^{*} )   +  T\eta^{\frac{3-\alpha}{1-\alpha}} \big) \Big)^{\frac{2\alpha}{1+\alpha}} \nonumber\\
      &\qquad +  \gamma^{\frac{1+\alpha }{\alpha-1}} + \gamma T \eta^{\frac{2}{1-\alpha}} + \frac{1}{T\eta}  + \eta \sigma^2 d   \Bigg). 
      \end{align}

Now, we can prove part (a) by choosing suitable $\gamma$, $\eta$ and $T$. Let $\gamma=\sqrt{n}$ and $\eta=c\min\Big\{ \frac{1}{\sqrt{n}}, \frac{\epsilon}{\sqrt{d\log(1/\delta)}} \Big\}$. Recall that $\sigma^2 d=\O\Big( \frac{Td\log(1/\delta)}{n^2 \epsilon^2} \Big)$. Note we assume $\eta T\ge 1$. Then  
\[ \eta      +   
      T\eta^3\sigma^2d + T\eta^2+ T\eta^{\frac{3-\alpha}{1-\alpha}} =\O\Big( \frac{T^2\eta^3 d \log(1/\delta)}{n^2\epsilon^2} + T\eta^2 \Big)=\O\Big(  {T^2{n^{-2}}\eta} + T\eta^2 \Big).  \]
Combining the above equation with Eq.\eqref{eq:excess-nonsmooth-5}, we get
\begin{align*} 
 \E_{S,\A}[  F(\bw_{\priv})] - F(\bw^{*})    = \O\Bigg( & \Big(n^{-\frac{1}{2}}T^{\frac{-2\alpha}{1+\alpha}} \eta^{\frac{ -4\alpha}{1+\alpha}} +    n^{-\frac{3}{2}}   T^{\frac{2}{1+\alpha}} \eta^{\frac{2-2\alpha}{1+\alpha}} + T^{\frac{-2\alpha }{1+\alpha}} \eta^{\frac{1-3\alpha}{1+\alpha}}\Big) \Big( T^2{n^{-2}}\eta + T\eta^2 \Big)^{\frac{2\alpha}{1+\alpha}}  \nonumber\\
      &+  n^{\frac{1+\alpha }{2(\alpha-1)}} + \frac{1}{\sqrt{n}} +   \frac{\sqrt{d\log(1/\delta)}}{n\epsilon}  \Bigg),
\end{align*}
 If we further choose $T\asymp n$, then for any $\alpha \in[1/2,1)$ there holds
\begin{align*} 
 \E_{S,\A}[  F(\bw_{\priv})] - F(\bw^{*}) = \O\Bigg(   \frac{1}{\sqrt{n}}  + \frac{\sqrt{d\log(1/\delta)}}{n\epsilon}  \Bigg).
\end{align*}

 For the case $\alpha\in[0,1/2)$, let $\gamma=\sqrt{n}$ and $\eta=c \min\Big\{ n^{\frac{3(\alpha-1)}{2(1+\alpha)}}, \frac{\epsilon}{\sqrt{d\log(1/\delta)}} \Big\}\le \min\{2/L,  1\}$ for some constant $c>0$. Similar to the discussion of Part (a), this choice of $\eta$ implies \[ \eta      +   
      T\eta^3\sigma^2d+T\eta^2 +  T\eta^{\frac{3-\alpha}{1-\alpha}} =\O\Big( \frac{T^2\eta^3 d \log(1/\delta)}{n^2\epsilon^2} + T\eta^2\Big)=\O\Big( \frac{T^2\eta^2\sqrt{ d \log(1/\delta)}}{n ( n \epsilon )} + T\eta^2 \Big).  \]
Further setting $T\asymp n^{\frac{2-\alpha}{1+\alpha}}$, then combining  the above equation with Eq.\eqref{eq:excess-nonsmooth-5} implies 
\begin{align*} 
 \E_{S,\A}[  F(\bw_{\priv})]\!-\! F(\bw^{*})   & = \O\Bigg( \frac{n^{\frac{-5\alpha^2 + 4\alpha -3}{2(1+\alpha)^2}}  \sqrt{ d \log(1/\delta)}}{  n \epsilon  } + n^{\frac{1+\alpha}{2(\alpha-1)}}
       +   \frac{\sqrt{d\log(1/\delta)}}{n^{\frac{2-\alpha}{1+\alpha}}\epsilon} +\frac{1}{\sqrt{n}} + \frac{\sqrt{d\log(1/\delta)}}{n\epsilon}  \Bigg)\\
  & = \O\Bigg( \frac{1}{\sqrt{n}} 
      + \frac{\sqrt{d\log(1/\delta)}}{n\epsilon}  \Bigg),
\end{align*}
where the last equality used  $\alpha< 1/2$. The proof of part (a) is completed.

 Finally,  we consider the low noise case, i.e.,  $F(\bw^{*})=0$. Let  $\eta = c \min\Big\{n^{\frac{\alpha^2+ 2\alpha-3}{2(1+\alpha)}}, \frac{n\epsilon}{T\sqrt{d\log(1/\delta)}}\Big\}\le \min\{2/L, 1\} $. 
Then \eqref{eq:excess-nonsmooth-5} implies
\begin{align*}
 \E_{S,\A}[  F(\bw_{\priv})] - F(\bw^{*})& = \O\Bigg(   \Big(\gamma^{-1}T^{\frac{-2\alpha}{1+\alpha}} \eta^{\frac{ -4\alpha}{1+\alpha}} + \gamma   n^{-2}   T^{\frac{2}{1+\alpha}} \eta^{\frac{2-2\alpha}{1+\alpha}} + T^{\frac{-2\alpha }{1+\alpha}} \eta^{\frac{1-3\alpha}{1+\alpha}}\Big) \Big(  \eta      +   
      T\eta^3\sigma^2 d  + T\eta^{\frac{3-\alpha}{1-\alpha}} \big) \Big)^{\frac{2\alpha}{1+\alpha}} \nonumber\\
      &\qquad +  \gamma^{\frac{1+\alpha }{\alpha-1}} + \gamma T \eta^{\frac{2}{1-\alpha}} + \frac{1}{T\eta}  + \eta \sigma^2 d   \Bigg).
\end{align*}
Note for any $\alpha\in[0,1)$, there holds
\[ \eta      +   
      T\eta^3\sigma^2 d  +  T\eta^{\frac{3-\alpha}{1-\alpha}} =\O\Big( \eta \Big( 1 + \frac{T^2\eta^2 d \log(1/\delta)}{n^2\epsilon^2} \Big) \Big)=\O(\eta), \]
      where we used $T^2\eta^2=\O\big(n^2\epsilon^2/(d\log(1/\delta))\big)$. 
Further, if we choose $\gamma= n^{\frac{1-\alpha}{2}}$ and $T\asymp n^{\frac{2}{1+\alpha}}$,
there holds
\begin{align*}
 \E_{S,\A}[  F(\bw_{\priv})] - F(\bw^{*})    = \O\Big( \frac{1}{n^{\frac{1+\alpha}{2}}} + \frac{\sqrt{d \log(1/\delta)}}{n\epsilon} \Big),
\end{align*}
which completes the proof.
\end{proof}

\subsection{Proofs for  Pairwise Learning}\label{appendix-pairwise}
We now turn to the analysis of  DP-SGD for pairwise learning algorithm (i.e. Algorithm~\ref{alg2}) and provide the proofs for
Theorems \ref{thm:excess-pair} and \ref{thm:excess-nonsmooth-pair}. 

We start with the proof of Theorem~\ref{thm:dp-pairwise}.  Specifically, we first prove that each iteration $t$ of the algorithm satisfies RDP by applying  Lemma~\ref{lem:uniform-rdp} with sampling rate $2/p$. Then according to  Lemma~\ref{lem:composition_RDP}  and  Lemma~\ref{lemma:RDP_to_DP}, we can show that the proposed algorithm satisfies $(\epsilon,\delta)$-DP. The detailed proof is shown as follows. 
\begin{proof}[Proof of Theorem~\ref{thm:dp-pairwise}]\label{proof:pairwise-DP}
For each $t\in[T]$, we consider  the mechanism  $\mathcal{A}_t=\mathcal{M}_t+\bb_t$, where $\mathcal{M}_t=\partial f(\bw_t;z_{i_t}, z_{j_t})$. Similar to before, we can show that the $\ell_2$-sensitivity of $\mathcal{M}_t$ is 2G by using Lipschitz continuity of $f$. Notice  that
\[ \sigma^2=\frac{56G^2T}{\beta n^2 \epsilon}\Big( \frac{\log(1/\delta)}{(1-\beta)\epsilon}+1 \Big). \]
Note that  $z_{i_t}$ and $z_{j_t}$ are drawn  uniformly without replacement from the training set $S$. Then
according to Lemma~\ref{lem:uniform-rdp} with $p=2/n$, we know $\mathcal{\A}_t$ satisfies  $\Big(\lambda, \frac{\lambda \beta \epsilon}{T\big( \frac{\log(1/\delta)}{(1-\beta)\epsilon} +1\big)} \Big)$-RDP as long as $\sigma^2\ge 2.68G^2$ and $\lambda-1\le \frac{\sigma^2}{6G^2}\log\Big( \frac{n}{2\lambda\big(1+\frac{\sigma^2}{4G^2}\big)} \Big)$ hold. 
Now, 
let $\lambda=\frac{\log(1/\delta)}{(1-\beta)\epsilon}+1$. Then we get  $\A_t$ satisfies  $\Big(\frac{\log(1/\delta)}{(1-\beta)\epsilon}+1,\frac{\beta\epsilon}{T}\Big)$-RDP. According to   Lemma~\ref{lemma:post-processing} and Lemma~\ref{lem:composition_RDP}, we can show that Algorithm~\ref{alg2} is $\Big(\frac{\log(1/\delta)}{(1-\beta)\epsilon}+1, \beta\epsilon \Big)$-RDP. Finally,  Lemma~\ref{lemma:RDP_to_DP}  implies Algorithm~\ref{alg2} is $(\epsilon,\delta)$-DP if $\sigma^2\ge 2.68G^2$ and $\lambda-1\le \frac{\sigma^2}{6G^2}\log\Big( \frac{n}{\lambda\big(1+\frac{\sigma^2}{4G^2}\big)} \Big)$ hold. The proof is completed.
\end{proof}

To establish the generalization analysis of Algorithm~\ref{alg2}, we first introduce the connection between stability and generalization error in the following lemma.    
\begin{lemma}[Generalization via stability for pairwise learning]\label{lem:gen-pair}
Let $\A$ be on-average $\nu$-argument stable. Let $\gamma>0$. 
\begin{enumerate}[label=({\alph*})]
\item If $f$ is nonnegative and $L$-smooth, then
\[ \E_{S,\A}[\bar{F}(\A(S))-\bar{F}_S(\A(S))]\le \frac{L}{\gamma}\E_{S,\A}[ \bar{F}_S(\A(S)) ] +  2(L+\gamma)\nu .  \]
\item If $f$ is nonnegative, convex and $\alpha$-H\"older smooth with parameter $L$ and $\alpha\in[0,1)$, then 
\[ \E_{S,\A}[\bar{F}(\A(S))-\bar{F}_S(\A(S))]\le \frac{c_{\alpha,1}^2}{2\gamma  }   \E_{S,\A}\big[ \bar{F}^{\frac{2\alpha}{1+\alpha}}(\A(S))\big] +  2\gamma \nu   . \]
\end{enumerate}
\end{lemma}
\begin{proof} 
 Part (a) was established in \cite{lei2021generalization}. 
 We only consider Part (b). Recall that $S=\{z_1,\ldots,z_n\}$ and $S'=\{z_1',\ldots,z_n'\}$ are drawn independently from $\rho$. For any $i\in[n]$, denote $S^{(i)}=\{ z_1,\ldots,z_{i-1},z_i',z_{i+1},\ldots,z_n \}$. Further, let \[S^{(i,j)}=\{z_1,\ldots,z_{i-1},z'_i,z_{i+1},\ldots, z_{j-1},z'_j,z_{j+1}, \ldots,z_n\}.\]  According to the symmetry between $z_i,z_j$ and $z'_i,z'_j$, we have
 \begin{align}\label{eq:gen-pair-1}
     &\E_{S,S',\A}[\bar{F}(\A(S))-\bar{F}_S(\A(S))]\nonumber\\
     &=\frac{1}{n(n-1)} \sum_{i,j\in[n]:i\neq j}    \E_{S,S',\A}[\bar{F}(\A(S^{(i,j)}))-\bar{F}_S(\A(S))]\nonumber\\
     &=\frac{1}{n(n-1)} \sum_{i,j\in[n]:i\neq j}    \E_{S,S',\A}[f(\A(S^{(i,j)};z_i,z_j))-f(\A(S);z_i,z_j)]\nonumber\\
     &\le \frac{1}{n(n-1)} \sum_{i,j\in[n]:i\neq j}    \E_{S,S',\A}[ \langle \partial f(\A(S^{(i,j)};z_i,z_j)), \A(S^{(i,j)})-\A(S) \rangle ],
 \end{align}
 where in the second equality we used $\E_{z_i,z_j}[ f(\A(S^{(i,j)});z_i,z_j) ]=\bar{F}(\A(S^{(i,j)}))$ since $z_i,z_j$ are independent of $\A(S^{(i,j)})$, and in the last inequality we used the convexity of $f$.  
 
 By the Schwartz’s inequality and self-bounding property (Lemma~\ref{lem:self-bounding}) we know
 \begin{align*}
     &\langle \partial f(\A(S^{(i,j)};z_i,z_j)), \A(S^{(i,j)})-\A(S) \rangle\\ &\le  \frac{1}{2\gamma}\|\partial f(\A(S^{(i,j)};z_i,z_j))\|_2^2+\frac{\gamma}{2}\| \A(S^{(i,j)})-\A(S) \|_2^2 \\
     &\le   \frac{c_{\alpha,1}^2}{2\gamma} f^{\frac{2\alpha}{1+\alpha}}( \A(S^{(i,j)};z_i,z_j) +  \gamma \| \A(S^{(i,j)})-\A(S^i) \|_2^2 + \gamma \| \A(S^{(i)})-\A(S) \|_2^2 .
 \end{align*}
 Plugging the above inequality back into Eq.\eqref{eq:gen-pair-1} we get
 \begin{align*} 
     &\E_{S,S',\A}[\bar{F}(\A(S))-\bar{F}_S(\A(S))]\nonumber\\
     &\le \frac{1}{n(n-1)} \sum_{i,j\in[n]:i\neq j}    \E_{S,S',\A}\Big[  \frac{c_{\alpha,1}^2}{2\gamma} f^{\frac{2\alpha}{1+\alpha}}( \A(S^{(i,j)};z_i,z_j) +\gamma \| \A(S^{(i,j)})-\A(S^i) \|_2^2 + \gamma \| \A(S^{(i)})-\A(S) \|_2^2  \Big]\\
     &=  \frac{c_{\alpha,1}^2}{2\gamma n(n-1)} \sum_{i,j\in[n]:i\neq j}    \E_{S,S',\A}\big[  f^{\frac{2\alpha}{1+\alpha}}( \A(S^{(i,j)};z_i,z_j) \big] + \frac{2\gamma}{n(n-1)} \sum_{i,j\in[n]:i\neq j}    \E_{S,S',\A}\big[  \| \A(S^{(i )})-\A(S ) \|_2^2  \big] ,
 \end{align*}
 where the last equality is due to $\E_{S,S',\A}\big[  \| \A(S^{(i,j )})-\A(S^i ) \|_2^2  \big]=\E_{S,S',\A}\big[  \| \A(S^{(j )})-\A(S ) \|_2^2  \big]$. 
 
Since $x\mapsto x^{\frac{2\alpha}{1+\alpha}}$ is concave and $z_i,z_j$ are independent of $\A(S^{(i,j)})$, we know
\begin{align*}
    \E_{S,S',\A}\big[  f^{\frac{2\alpha}{1+\alpha}}( \A(S^{(i,j)};z_i,z_j) \big]&\le \E_{S,S',\A}\big[ \big( \E_{z_i,z_j}[ f( \A(S^{(i,j)};z_i,z_j) ]\big)^{\frac{2\alpha}{1+\alpha}}\big] \\
    &=\E_{S,\A}\big[ \bar{F}^{\frac{2\alpha}{1+\alpha}}(\A(S))\big] .
\end{align*}
Combining the above two inequalities together implies
 \begin{align*} 
     &\E_{S,S',\A}[\bar{F}(\A(S))-\bar{F}_S(\A(S))]\nonumber\\
     &=  \frac{c_{\alpha,1}^2}{2\gamma n(n-1)} \sum_{i,j\in[n]:i\neq j}    \E_{S,\A}\big[ \bar{F}^{\frac{2\alpha}{1+\alpha}}(\A(S))\big] + \frac{2\gamma}{n(n-1)} \sum_{i,j\in[n]:i\neq j}    \E_{S,S',\A}[  \| \A(S^{(i )})-\A(S ) \|_2^2  ] \\
     &= \frac{c_{\alpha,1}^2}{2\gamma  }   \E_{S,\A}\big[ \bar{F}^{\frac{2\alpha}{1+\alpha}}(\A(S))\big] + \frac{2\gamma}{n } \sum_{i=1 }^n    \E_{S,S',\A}[  \| \A(S^{(i )})-\A(S ) \|_2^2  ] . 
 \end{align*}
 The proof of Part (b) is completed. 
\end{proof}

Our stability analysis for $\alpha$-H\"older smooth losses requires the following lemma,  which shows the approximately non-expansive behavior of the gradient mapping $\bw\mapsto \bw-\eta\partial f(\bw;z,z')$.

\begin{lemma}[\cite{lei2020fine}]\label{lem:non-expensive}
Assume for all $z,z'\in\Z$, the map $\bw\mapsto f(\bw;z,z')$ is convex, and $\bw \mapsto \partial f(\bw;z,z')$ is $\alpha$-H\"older smooth with parameter $L$ and $\alpha\in[0,1)$. Then for all $\bw,\bw'$ and $\eta >0$ we have
\[ \|\bw-\eta \partial f(\bw;z,z')-\bw'+\eta \partial f(\bw';z,z')\|_2^2\le \|\bw-\bw'\|_2^2+c_{\alpha,3}^2\eta^{\frac{2}{1-\alpha}}. \]
\end{lemma}
 
As discussed in Section~\ref{appendix-pointwise}, adding noise to gradient will not impact stability results. Hence,   we  only need to address  the on-average stability bounds of non-private SGD for pairwise learning. 
 \begin{lemma}[Stability bounds]\label{lem:stability-pairwise}
 Suppose $f$ is nonnegative and convex.  Let $S,S' $ and $S^{(i)}$ be constructed as Definition~\ref{def:avg-stability}. Let $\{\bw_{t}\}$ and $\{\bw_t^{(i)}\}$ be produced by Algorithm~\ref{alg2} based on $S$ and $S^{(i)}$, respectively.
 \begin{enumerate}[label=({\alph*})]
\item If $f$ is $L$-smooth and $\eta_t\le 2/L$ for all $t\in[T]$, then
     \[  \E_{S,S',\A}\Big[\frac{1}{n}\sum_{i=1}^n\| \bw_{t+1}-\bw_{t+1}^{(i)} \|_2^2\Big]\le \frac{16L(1+2t/n)e}{n} \sum_{j=1}^t  \eta_j^2 \E_{S,\A}[ F_S(\bw_j) ].  \]
\item If $f$ is $\alpha$-H\"older smooth with parameter $L$ and $\alpha\in[0,1)$, then
\begin{align*}
    \E_{S,S',\A}\Big[\frac{1}{n}\sum_{i=1}^n\| \bw_{t+1}-\bw_{t+1}^{(i)} \|_2^2\Big]\le \frac{8ec_{\alpha,1}^2(1+2t/n) }{n}\sum_{j=1}^t \eta_j^2\E_{S,\A}\Big[ F_S^{\frac{2\alpha}{1+\alpha}} (\bw_j)\Big] \Big) + c_{\alpha,3}^2e \sum_{j=1}^t \eta_j^{\frac{2}{1-\alpha}} ,
\end{align*}
where $c_{\alpha,3}= \sqrt{ \frac{e(1-\alpha)}{1+\alpha}}(2^{-\alpha}L)^{\frac{1}{1-\alpha}}$. 
 \end{enumerate}

 \end{lemma}

\begin{proof} 
The proof of part (a) can be found in \cite{lei2021generalization}. We only give the proof of part (b).
 For any $i\in[n]$, let $S,S^{(i)}$ and $S'$ be   constructed as Definition~\ref{def:avg-stability}. For any $S$ and $i\in[n]$, we consider the following three cases.
 
 \noindent\textbf{Case 1.} If $i_t \neq i$ and $j_t\neq i$, it then follows from the update rule of $\bw_{t+1}$ and Lemma~\ref{lem:non-expensive} that
 \begin{align*}
     \|\bw_{t+1 }-\bw^{(i)}_{t+1}\|_2^2&\le  \|\bw_{t } -\eta _t \partial f(\bw_{t};z_{i_t},z_{j_t})-\bw^{(i)}_{t}+ \eta _t \partial f(\bw^{(i)}_{t};z_{i_t},z_{j_t})\|_2^2\\
     &\le  \|\bw_{t }-\bw^{(i)}_{t} \|_2^2 +c_{\alpha,3}^2\eta_t^{\frac{2}{\alpha-1}}.
 \end{align*}

\noindent\textbf{Case 2.} If $i_t = i$, it then follows from the update rule and the standard inequality $(a+b)^2\le (1+p)a^2+(1+1/p)b^2$  that
  \begin{align*}
       \|\bw_{t+1 }-\bw^{(i)}_{t+1}\|_2^2
       &\le (1+p) \|\bw_{t }   -\bw^{(i)}_{t}\|_2^2 + (1+1/p)\eta_t^2\big(\| \partial f(\bw_{t};z_{i },z_{j_t})-  \partial f(\bw^{(i)}_{t};z'_{i },z_{j_t})\|_2^2\big)\\
       &\le (1+p) \|\bw_{t }  -\bw^{(i)}_{t}\|_2^2 + 2(1+1/p)\eta_t^2\big(\| \partial f(\bw_{t};z_{i },z_{j_t})\|_2^2+\|  \partial f(\bw^{(i)}_{t};z'_{i },z_{j_t})\|_2^2\big)\\
       &\le (1+p) \|\bw_{t }  -\bw^{(i)}_{t}\|_2^2 + 2c_{\alpha,1}^2(1+1/p)\eta_t^2\big(  f^{\frac{2\alpha}{1+\alpha}}(\bw_{t};z_{i },z_{j_t})+  f^{\frac{2\alpha}{1+\alpha}}(\bw^{(i)}_{t};z'_{i },z_{j_t}) \big). 
  \end{align*}

\noindent\textbf{Case 3.} If $j_t = i$, similar to Case 2, we have
  \begin{align*}
       \|\bw_{t+1 }-\bw^{(i)}_{t+1}\|_2^2
       &\le (1+p) \|\bw_{t }  -\bw^{(i)}_{t }\|_2^2 + 2c_{\alpha,1}^2(1+1/p)\eta_t^2\big(  f^{\frac{2\alpha}{1+\alpha}}(\bw_{t};z_{i_t },z_{i})+  f^{\frac{2\alpha}{1+\alpha}}(\bw^{(i)}_{t};z_{i_t },z'_{i}) \big).
  \end{align*} 
Note $\textbf{Pr}(i_t\neq i \text{ and } j_t\neq i)=\frac{(n-1)(n-2)}{n(n-1)}$ and $\textbf{Pr}(i_t= i \text{ and } j_t= j )=\frac{1}{n(n-1)}$  for any $j\neq i$. We can combine the above three cases together and get
  \begin{align*}
       &\E_{i_t,j_t}\big[\|\bw_{t+1 }-\bw^{(i)}_{t+1}\|_2^2
       \big]\nonumber\\
       &\le \frac{(n-1)(n-2)}{n(n-1)} \Big(  \|\bw_{t }-\bw^{(i)}_{t} \|_2^2 +c_{\alpha,3}^2\eta_t^{\frac{2}{\alpha-1}} \Big)\nonumber\\
       &\quad +  \frac{1}{n(n-1)}\sum_{j\in[n]: j\neq i}\Big( (1+p) \|\bw_{t }  -\bw^{(i)}_{t}\|_2^2 + 2c_{\alpha,1}^2(1+1/p)\eta_t^2\big(  f^{\frac{2\alpha}{1+\alpha}}(\bw_{t};z_{i },z_{j })+  f^{\frac{2\alpha}{1+\alpha}}(\bw^{(i)}_{t};z'_{i },z_{j }) \big)\Big)\nonumber\\
       &\quad + \frac{1}{n(n-1)}\sum_{j\in[n]: j\neq i} \Big((1+p) \|\bw_{t }  -\bw^{(i)}_{t }\|_2^2 + 2c_{\alpha,1}^2(1+1/p)\eta_t^2\big(  f^{\frac{2\alpha}{1+\alpha}}(\bw_{t};z_{j },z_{i})+  f^{\frac{2\alpha}{1+\alpha}}(\bw^{(i)}_{t};z_{j},z'_{i}) \big)\Big)\nonumber\\
       &\le \Big(1+\frac{2p}{n}\Big)  \|\bw_{t }-\bw^{(i)}_{t} \|_2^2 + c_{\alpha,3}^2\eta_t^{\frac{2}{\alpha-1}}   + \frac{2(1+1/p)c^2_{\alpha,1}\eta _t^2}{n(n-1)} \sum_{j\in[n]: j\neq i}\Big[ f^{\frac{2\alpha}{1+\alpha}}(\bw_{t};z_{i },z_{j})+  f^{\frac{2\alpha}{1+\alpha}}(\bw^{(i)}_{t};z'_{i },z_{j})\nonumber\\
       &\quad +   f^{\frac{2\alpha}{1+\alpha}}(\bw_{t};z_{j },z_{i})+  f^{\frac{2\alpha}{1+\alpha}}(\bw^{(i)}_{t};z_{j },z'_{i})\Big].
  \end{align*} 
  Taking an average over  $i$ we have
    \begin{align*}
       &\frac{1}{n}\sum_{i=1}^n\E_{i_t,j_t}\big[\|\bw_{t+1 }-\bw^{(i)}_{t+1}\|_2^2
       \big]\nonumber\\
       &\le \Big(1+\frac{2p}{n}\Big) \frac{1}{n}\sum_{i=1}^n \|\bw_{t }-\bw^{(i)}_{t} \|_2^2  + c_{\alpha,3}^2\eta_t^{\frac{2}{\alpha-1}}   + \frac{2(1+1/p)c^2_{\alpha,1}\eta _t^2}{n^2(n-1)}  \sum_{i=1}^n \sum_{j\in[n]: j\neq i}\Big[ f^{\frac{2\alpha}{1+\alpha}}(\bw_{t};z_{i },z_{j}) \nonumber\\
       &\quad+  f^{\frac{2\alpha}{1+\alpha}}(\bw^{(i)}_{t};z'_{i },z_{j}) + f^{\frac{2\alpha}{1+\alpha}}(\bw_{t};z_{j },z_{i})+  f^{\frac{2\alpha}{1+\alpha}}(\bw^{(i)}_{t};z_{j },z'_{i})\Big].
  \end{align*} 
  Further, taking an expectation over both sides yields
      \begin{align*}
       &\frac{1}{n}\sum_{i=1}^n\E_{S,S',\A}\big[\|\bw_{t+1 }-\bw^{(i)}_{t+1}\|_2^2
       \big]\nonumber\\
       &\le \Big(1+\frac{2p}{n}\Big) \frac{1}{n}\sum_{i=1}^n \E_{S,S',\A}\big[\|\bw_{t }-\bw^{(i)}_{t} \|_2^2\big]   + \frac{2(1+1/p)c^2_{\alpha,1}\eta _t^2}{n^2(n-1)}\sum_{i=1}^n \E_{S,S',\A}\Big[  \sum_{j\in[n]: j\neq i}\Big[ f^{\frac{2\alpha}{1+\alpha}}(\bw_{t};z_{i },z_{j})\nonumber\\
       &\quad +  f^{\frac{2\alpha}{1+\alpha}}(\bw^{(i)}_{t};z'_{i },z_{j})+ f^{\frac{2\alpha}{1+\alpha}}(\bw_{t};z_{j },z_{i})+  f^{\frac{2\alpha}{1+\alpha}}(\bw^{(i)}_{t};z_{j },z'_{i})\Big]\Big] + c_{\alpha,3}^2\eta_t^{\frac{2}{\alpha-1}} .
  \end{align*}
  Due to the symmetry between $z_i$ and $z'_i$ we know
  \[\E_{S ,\A}\Big[  \!\sum_{j\in[n]: j\neq i}\Big[ f^{\frac{2\alpha}{1+\alpha}}(\bw_{t};z_{i },z_{j}) + f^{\frac{2\alpha}{1+\alpha}}(\bw_{t};z_{j  },z_{i}) \Big] \Big]   =  \E_{S,S',\A}\Big[ \!\sum_{j\in[n]: j\neq i}\Big[ f^{\frac{2\alpha}{1+\alpha}}(\bw^{(i)}_{t};z'_{i },z_{j}) + f^{\frac{2\alpha}{1+\alpha}}(\bw^{(i)}_{t};z_{j },z'_{i})\Big] \Big]. \]
  It then follows that
       \begin{align*}
       &\frac{1}{n}\sum_{i=1}^n\E_{S,S',\A}\big[\|\bw_{t+1 }-\bw^{(i)}_{t+1}\|_2^2
       \big]\nonumber\\
       &\le \Big(1+\frac{2p}{n}\Big) \frac{1}{n}\sum_{i=1}^n \E_{S,S',\A}\big[\|\bw_{t }-\bw^{(i)}_{t} \|_2^2\big]   + \frac{4(1+1/p)c^2_{\alpha,1}\eta _t^2}{n^2(n-1)}\sum_{i=1}^n \E_{S ,\A}\Big[  \sum_{j\in[n]: j\neq i}\Big[ f^{\frac{2\alpha}{1+\alpha}}(\bw_{t};z_{i },z_{j})\nonumber\\
       &\quad +    f^{\frac{2\alpha}{1+\alpha}}(\bw_{t};z_{j },z_{i}) \Big]\Big] + c_{\alpha,3}^2\eta_t^{\frac{2}{\alpha-1}} \nonumber\\
       &= \Big(1+\frac{2p}{n}\Big) \frac{1}{n}\sum_{i=1}^n \E_{S,S',\A}\big[\|\bw_{t }-\bw^{(i)}_{t} \|_2^2\big]  + \frac{8(1+1/p)c^2_{\alpha,1}\eta _t^2}{n}  \E_{S ,\A}\Big[ \frac{1}{n(n-1)}\sum_{i=1}^n  \sum_{j\in[n]: j\neq i}  f^{\frac{2\alpha}{1+\alpha}}(\bw_{t};z_{i },z_{j}) \Big]\\
       &\quad +   c_{\alpha,3}^2\eta_t^{\frac{2}{\alpha-1}}  ,
  \end{align*}
  where in the last equality we used $\sum_{i=1}^n  \sum_{j\in[n]: j\neq i}  f^{\frac{2\alpha}{1+\alpha}}(\bw_{t};z_{j},z_{i}) = \sum_{i=1}^n  \sum_{j\in[n]: j\neq i}  f^{\frac{2\alpha}{1+\alpha}}(\bw_{t};z_{i },z_{j}) $. 
  
Further,  according  to Jensen's inequality and $\bw_1=\bw'_1$, we know
       \begin{align*}
       \frac{1}{n}\sum_{i=1}^n\E_{S,S',\A}\big[\|\bw_{t+1 }-\bw^{(i)}_{t+1}\|_2^2
       \big] \le &  \Big(1+\frac{2p}{n}\Big) \frac{1}{n}\sum_{i=1}^n \E_{S,S',\A}\big[\|\bw_{t }-\bw^{(i)}_{t} \|_2^2\big]  + c_{\alpha,3}^2\eta_t^{\frac{2}{\alpha-1}} \\
       & + \frac{8(1+1/p)c^2_{\alpha,1}\eta _t^2}{n}  \E_{S ,\A}\Big[   F^{\frac{2\alpha}{1+\alpha}}_S(\bw_t) \Big],
  \end{align*}
  Now, we can apply the above inequality recursively and get
      \begin{align*}
        \frac{1}{n}\sum_{i=1}^n\E_{S,S',\A}\big[\|\bw_{t+1 }-\bw^{(i)}_{t+1}\|_2^2
       \big] \le & \frac{8(1+1/p)c^2_{\alpha,1}}{n} \sum_{j=1}^t \Big(1+\frac{2p}{n}\Big)^{t-j} \eta_t^2 \E_{S ,\A}\Big[   F^{\frac{2\alpha}{1+\alpha}}_S(\bw_j) \Big]
      \\
      & +
       c_{\alpha,3}^2 \sum_{j=1}^t \Big(1+\frac{2p}{n}\Big)^{t+1-j} \eta_t^{\frac{2}{\alpha-1}}  .
  \end{align*}
  Finally, we can set $p=\frac{n}{2t}$ and use $(1+1/t)^t \le e$ to get
       \begin{align*}
        \frac{1}{n}\sum_{i=1}^n\E_{S,S',\A}\big[\|\bw_{t+1 }-\bw^{(i)}_{t+1}\|_2^2
       \big] \le & \frac{8e(1+2t/n)c^2_{\alpha,1}}{n} \sum_{j=1}^t   \eta_t^2 \E_{S ,\A}\Big[   F^{\frac{2\alpha}{1+\alpha}}_S(\bw_j)  \Big]
     +
       c_{\alpha,3}^2e \sum_{j=1}^t   \eta_t^{\frac{2}{\alpha-1}} ,
  \end{align*}
which completes the proof.
\end{proof}

To prove Theorem~\ref{thm:excess-pair}, we introduce the following lemma on optimization error. As discussed in \cite{lei2021generalization}, the optimization error analysis of DP-SGD (Algorithm~\ref{alg2}) for pairwise learning is the same as that for pointwise learning (Algorithm~\ref{alg1}). Here, $\alpha=1$ corresponds to the strongly smooth case due to the definition of $\alpha$-H\"older smoothness.
\begin{lemma}\label{lem:opt-pairwise}
Suppose $f$ is nonnegative, convex and $\alpha$-H\"older smooth with parameter $L$ and $\alpha\in[0,1]$. Let $\{\bw_t\}$ be produced by Algorithm~\ref{alg2} with $\eta_t=\eta$. Then
\begin{align*} 
     & \sum_{j=1}^t \eta_j \E_{\A}[\bar{F}_S(\bw_j ) - \bar{F}_S(\bw^{*} )] \\
    &\le \frac{1}{2} \|  \bw^{*} \|_2^2  + \frac{3}{4}c_{\alpha,1}^2\Big(\sum_{j=1}^t \eta_j^2\Big)^{\frac{1-\alpha}{1+\alpha}} \Big[2\eta_1  \| \bw^{*} \|_2^2     +   
     \sum_{j=1}^t \big(6\eta_j^3\sigma^2d + 4\eta_j^2  \bar{F}_S(\bw^{*} )   + 3c_{\alpha,2}\eta_j^{\frac{3-\alpha}{1-\alpha}}  \big) \Big]^{\frac{2\alpha}{1+\alpha}}  +  \sum_{j=1}^t  3\eta_j^2\sigma^2 d  
\end{align*}
and
\begin{align*} 
    \sum_{j=1}^t\eta_j^2 \E_{S,\A}[\bar{F}_S(\bw_t )  ]
    &\le 2\eta_1  \| \bw^{*} \|_2^2     +   
     \sum_{j=1}^t \big(6\eta_j^3\| \bb_j \|_2^2 + 4\eta^2_j  \bar{F}(\bw^{*} )   + 3c_{\alpha,2}\eta_j^{\frac{3-\alpha}{1-\alpha}}  \big).
\end{align*}
\end{lemma}

\bigskip

Now, we are ready to prove the utility guarantees of Algorithm~\ref{alg2} for strongly smooth and non-smooth cases. We first present the proof for strongly smooth case (i.e., Theorem~\ref{thm:excess-pair}).  
\begin{proof}[Proof of Theorem~\ref{thm:excess-pair}]
 Similar to the proof of Theorem~\ref{thm:excess-smooth-point}, combining Lemma~\ref{lem:stability-pairwise},  Lemma~\ref{lem:opt-pairwise} and part (a) in Lemma~\ref{lem:gen-pair}  together we have 
 \begin{align*}
    \E_{S,\A}[\bar{F}(\bw_{t+1 }) ] \le& \Big(1+\frac{L}{\gamma}\Big) \E_{S,\A}[ \bar{F}_S(\bw_{t+1 }) ]\\
    &+ \frac{32e(L+\gamma)(1+2t/n)L}{ n} \Big[ 2\eta_1  \| \bw^{*} \|_2^2     +   
     \sum_{j=1}^t \big(6\eta_j^3\sigma^2 d + 4\eta^2_j \ \bar{F}(\bw^{*})   \big) \Big].
\end{align*}
Multiplying both sides by $\eta_{t+1}$ and taking a summation gives
\begin{align*} 
    \sum_{t=1}^T \eta_t\E_{S,\A}[\bar{F}(\bw_{t }) ] \le& \Big(1+\frac{L}{\gamma}\Big) \sum_{t=1}^T \eta_t \E_{S,\A}[ \bar{F}_S(\bw_{t }) ]\nonumber \\
    &+ \frac{32e(L+\gamma)(1+2T/n)L}{ n} \sum_{t=1}^T \eta_t \Big[ 2\eta_1  \| \bw^{*} \|_2^2     +   
     \sum_{j=1}^t \big(6\eta_j^3\sigma^2 d + 4\eta^2_j \ \bar{F}(\bw^{*})   \big) \Big].
\end{align*}
Lemma~\ref{lem:opt-pairwise} with $\alpha=1$ implies
\begin{align*} 
       \sum_{t=1}^T \eta_t \E_{\A}[\bar{F}_S(\bw_t ) ]  \le & \sum_{t=1}^T \eta_t \E_{\A}[ \bar{F}_S(\bw^{*} )]  + \frac{1}{2} \|  \bw^{*} \|_2^2 \\
    & +  {3L}   \Big( \eta_1  \| \bw^{*} \|_2^2     +   
     \sum_{t=1}^T \big(3\eta_t^3\sigma^2d + 2\eta_t^2  F_S(\bw^{*} )    \big) \Big)  +  \sum_{t=1}^T  3\eta_t^2\sigma^2 d. 
\end{align*}
Combining the above two inequalities together yields
\begin{align*}
 \sum_{t=1}^T \eta_t\E_{S,\A}[\bar{F}(\bw_{t }) ] \le& \Big(1\!+\!\frac{L}{\gamma}\Big)  \Big(\sum_{t=1}^T \eta_t \bar{F}(\bw^{*} )\!+\!\Big(\frac{1}{2}\!+\!3L  \eta_1 \Big) \| \bw^{*} \|_2^2   \!+\!3\sum_{j=1}^t \big(3L \eta_j\!+\!1\big)\eta_j^2\sigma^2 d\!+\!4 \sum_{j=1}^t \eta_j^2  \bar{F}(\bw^{*} )  \Big)  \nonumber \\
    &+ \frac{32e(L+\gamma)(1+2T/n)L}{ n} \sum_{t=1}^T \eta_t \Big[ 2\eta_1  \| \bw^{*} \|_2^2     +   
     \sum_{j=1}^t \big(6\eta_j^3\sigma^2 d + 4\eta^2_j \ \bar{F}(\bw^{*})   \big) \Big]. 
\end{align*}
Let $\eta_t=\eta\le \min\{2/L,1\}$ and assume $T\ge n$. Recall that $\sigma^2 d = \O\big(\frac{T d \log(1/\delta)}{n^2\epsilon^2})  \big)$. According to Jensen's inequality, there holds
\begin{align}\label{eq:excess-pair-1}
  \E_{S,\A}[\bar{F}(\bw_{\priv }) - \bar{F}(\bw^{*}) ] =\O\Bigg(& \Big( \frac{1}{\gamma} + \frac{ T^2  \eta^2(1+\gamma) }{n^2} + \big(\frac{1}{\gamma}+1\big)\eta \Big)\bar{F}(\bw^{*}) + 
\big( \frac{(1+\gamma^{-1})}{T\eta} + \frac{(1+\gamma)T\eta}{n^2} \Big)\|\bw^{*}\|_2^2\nonumber\\
&+ \Big(\big(1+\frac{1}{\gamma}\big)\eta + \frac{ T^2\eta^3(1+\gamma) }{n^2}\Big)\frac{Td\log(1/\delta)}{n^2\epsilon^2}\Bigg). 
\end{align}
  Now, we give the proof of part (a). We can set $T\asymp n$ , $\gamma=\sqrt{n}$  and $\eta_t=c/\max\Big\{ \sqrt{n}, \frac{\sqrt{d\log(1/\delta)}}{\epsilon} \Big\}\le \min\{2/L ,1\}$ for some constant $c>0$. Then from  Eq.\eqref{eq:excess-pair-1}  we obtain
\begin{align*}
  \E_{S,\A}[\bar{F}(\bw_{\priv }) - \bar{F}(\bw^{*}) ] =&\O\Big(  \frac{1}{\sqrt{n}}  + \frac{\sqrt{d\log(1/\delta)}}{n\epsilon}\Big),  
\end{align*}
where we also assume $\sqrt{d\log(1/\delta)}=\O(n\epsilon)$. 

\noindent (b) We now consider the low-noise case $F(\bw^{*})=0$. By setting $\gamma \ge 1$, $T\asymp n$ and $\eta_t=\frac{c \epsilon}{\sqrt{d\log(1/\delta)}}\le \min\{ 2/L,1\}$ for some constant $c>0$, we get  
\begin{align*}
  \E_{S,\A}[\bar{F}(\bw_{\priv }) - \bar{F}(\bw^{*}) ] =\O\Big(     \frac{\sqrt{ d\log(1/\delta)}}{n \epsilon }\Big), 
\end{align*}
which completes the proof. 
\end{proof}

Finally, we give the proof for Theorem~\ref{thm:excess-nonsmooth-pair}. 
\begin{proof}[Proof of Theorem~\ref{thm:excess-nonsmooth-pair}]
 The proof is similar to  that of Theorem~\ref{thm:excess-nonsmooth-point}. Specifically, we can plug part (b) in Lemma~\ref{lem:stability-pairwise} back into  part (b) in Lemma~\ref{lem:gen-pair}   to get that
 \begin{align}\label{eq:excess-pair-nonsmooth-1}
        &\big( \sum_{t=1}^T \eta_t\big)^{-1} \sum_{t=1}^T \eta_t \E_{S,\A}[ \bar{F}(\bw_t) - \bar{F}_S(\bw_t) ]\nonumber\\
    &=\O\Big( \gamma^{\frac{1+\alpha }{\alpha-1}} + \gamma T \eta^{\frac{2}{1-\alpha}} +  \big(\gamma  T  \eta \big)^{-1}\sum_{t=1}^T\eta  \big(\E_{S,\A}[ \bar{F}_S(\bw_{t+1})]\big) ^{\frac{2\alpha}{1+\alpha}}  + \gamma T n^{-2} \sum_{t=1}^T \eta^2 \big(\E_{S,\A}[ \bar{F}_S (\bw_t)]\big)^{\frac{2\alpha}{1+\alpha}} \Big).
 \end{align}
 Further, combining Eq.\eqref{eq:excess-pair-nonsmooth-1} and Lemma~\ref{lem:opt-pairwise} together we can obtain
 \begin{align} \label{eq:excess-nonsmooth-pair-2}
     &\big( \sum_{t=1}^T \eta \big)^{-1} \sum_{t=1}^T \eta  \E_{S,\A}[\bar{F}_S(\bw_t ) - \bar{F}\bw^{*} )]= \big( \sum_{t=1}^T \eta \big)^{-1} \sum_{t=1}^T \eta  \E_{S,\A}[\bar{F}_S(\bw_t ) - \bar{F}_S(\bw^{*} )]\nonumber\\ &=\O\Bigg( \frac{1}{T\eta} +   T^{\frac{-2\alpha }{1+\alpha}} \eta^{\frac{1-3\alpha}{1+\alpha}} \Big( \eta    +   
     T \eta^3\sigma^2d + T\eta^2  \bar{F}(\bw^{*} )   +  T\eta^{\frac{3-\alpha}{1-\alpha}}  \big) \Big)^{\frac{2\alpha}{1+\alpha}}  + \eta \sigma^2 d\Bigg). 
     \end{align}
Plugging Eq,\eqref{eq:excess-pair-nonsmooth-1} and Eq.\eqref{eq:excess-nonsmooth-pair-2} back into Eq.\eqref{eq:excess-nonsmooth-1} we have
\begin{align}\label{eq:excess-nonsmooth-pair-3}
  \E_{S,\A}[\bar{F}(\bw_{\priv }) - \bar{F}(\bw^{*}) ] =\O\Bigg(&  
\big( \frac{(1+\gamma^{-1})}{T\eta}\!+\!\frac{(1+\gamma)T\eta}{n^2} \Big)\|\bw^{*}\|_2^2\!+\!\Big( \gamma^{-1}\!+\!\frac{ T^2  \eta^2(1+\gamma) }{n^2}\!+\! \big(\gamma^{-1}\!+\!1\big)\eta \Big)\bar{F}(\bw^{*}) \nonumber\\
&+ \Big(\big(1\!+\!\gamma^{-1}\big)\eta\!+\!\frac{ T^2\eta^3(1+\gamma) }{n^2}\Big)\frac{Td\log(1/\delta)}{n^2\epsilon^2}\Bigg).
\end{align} 
The rest of the proof is similar to Theorem~\ref{thm:excess-nonsmooth-point}. We omit it for simplicity. 
\end{proof}

\section{Conclusion}\label{sec:conclu}

In this paper, we conducted a systematic analysis of DP-SGD with gradient perturbation for both pointwise and pairwise learning problems. For pointwise learning, we introduced a low-noise condition and derived sharper excess population risk bounds. Specifically, we achieved bounds in the order of $\O\big(\frac{1}{n\epsilon}\sqrt{d\log(1/\delta)}\big)$ and $\O\big( n^{-\frac{1+\alpha}{2}} + \frac{1}{n\epsilon}\sqrt{d\log(1/\delta)}\big)$ for strongly smooth and $\alpha$-H\"{o}lder smooth losses, respectively.

Regarding pairwise learning, we presented a computationally efficient DP-SGD algorithm with utility guarantees. Our analysis demonstrated that our algorithm achieves the optimal excess risk bounds of the order $\O\big( \frac{1}{\sqrt{n}} + \frac{1}{n\epsilon}\sqrt{d\log(1/\delta)}\big)$ for both strongly smooth and $\alpha$-H\"{o}lder smooth losses. Furthermore, we established faster excess risk bounds for both strongly smooth and $\alpha$-H\"{o}lder smooth losses under a low-noise condition. Notably, our work represents the first utility analysis for privacy-preserving pairwise learning that provides excess risk rates tighter than $\O\big( \frac{1}{\sqrt{n}} + \frac{1}{n\epsilon}\sqrt{d\log(1/\delta)}\big)$.

There are several open questions that remain for further study. Firstly, it would be interesting to explore whether our analysis of DP-SGD with uniform sampling can be extended to DP-SGD with Markov sampling, which poses a more challenging task. Secondly, an unexplored area for us is to investigate the utility analysis of DP-SGD with a neural network structure. Addressing these questions would contribute to a deeper understanding of privacy-preserving machine learning algorithms.

\medskip 

\noindent{\bf Acknowledgement.}  The work described in this paper is partially done when the last author, Ding-Xuan Zhou, worked at City University of Hong Kong, supported by
the Laboratory for AI-Powered Financial Technologies under the InnoHK scheme, the Research Grants Council of Hong Kong [Projects No. CityU 11308121, No. N\_CityU102/20, and No. C1013-21GF], the National Science Foundation of China [Project No. 12061160462], and the Hong Kong Institute for Data Science.  Yiming's work is supported by SUNY-IBM AI Alliance Research and NSF grants (IIS-2103450, IIS-2110546 and DMS-2110836)

\bibliographystyle{plain}
\bibliography{main.bib}

\end{document}